\documentclass[11pt]{article}
\usepackage[a4paper, total={6in, 8in}]{geometry}

\usepackage{times}

\usepackage{amsthm}
\usepackage{amssymb}
\usepackage{enumitem} 
\usepackage{amsmath}

\newtheorem{definition}{Definition}
\newtheorem{theorem}[definition]{Theorem}
\newtheorem{lemma}[definition]{Lemma}
\newtheorem{proposition}[definition]{Proposition}
\newtheorem{example}[definition]{Example}
\newtheorem{remark}[definition]{Remark}
\newtheorem{corollary}[definition]{Corollary}

\usepackage[thinlines]{easytable}
\usepackage{array}
\newcolumntype{P}[1]{>{\centering\arraybackslash}p{#1}} 
\newcolumntype{M}[1]{>{\centering\arraybackslash}m{#1}} 

\usepackage[bookmarks=true]{hyperref}
\hypersetup{%
    colorlinks=true,       
    linkcolor=blue,       
    citecolor=blue,       
    filecolor=black,        
    urlcolor=purple,        
    linktoc=page            
}

\usepackage{tikz}

\usepackage[labelfont=bf]{caption}

\usepackage{wrapfig}
\usepackage{graphicx}
\graphicspath{ {images/} }
\usepackage{subfig}

\usepackage{multirow}

\usepackage{orcidlink}

\def\dotminus{\mathbin{\ooalign{\hss\raise1ex\hbox{.}\hss\cr
  \mathsurround=0pt$-$}}}

\makeindex

\title{Selective Credibility-Limited Belief Update}

\author{
Theofanis Aravanis\textsuperscript{1,\,*}~\orcidlink{0000-0003-0329-3200} \and
Costas D. Koutras\textsuperscript{2}~\orcidlink{0000-0002-2488-0167}
}
\date{\small
\textsuperscript{1} Department of Digital Systems\\
School of Economics and Technology\\
University of the Peloponnese\\
Sparta 231 00, Greece\\
\href{mailto:taravanis@uop.gr}{\texttt{taravanis@uop.gr}}
\\[0.3em]
\textsuperscript{*}Corresponding author
\\[0.6em]
\textsuperscript{2} College of Engineering and Technology \\
American University of the Middle East\\
Kuwait\\
\href{mailto:Konstantinos.K@aum.edu.kw}{\texttt{Konstantinos.K@aum.edu.kw}}
}
\usepackage{fancyhdr}
\providecommand{\keywords}[1]{\textbf{\textit{Keywords:}} #1}

\begin{document}

\maketitle
\sloppy

\begin{abstract}
Belief update concerns changes in an agent's beliefs induced by changes in the underlying world. Standard Katsuno--Mendelzon update assumes that an epistemic input can be incorporated from every initially possible world, whereas credibility-limited belief update restricts, for each source world, the successor worlds regarded as credible or reachable. Nevertheless, existing credibility-limited approaches treat the epistemic input as an indivisible whole, and therefore cannot represent cases in which only part of a compound epistemic input can be realized. We introduce selective credibility-limited belief update, in which the epistemic input is transformed, relative to each source world, into a weaker proxy before the credibility-limited transition is performed. We provide semantic and axiomatic characterizations of the resulting class of update operators. We then identify two well-behaved sub-classes; namely, consistency-preserving update operators, which require every transformed epistemic input to be credible from its source world whenever the original epistemic input is consistent, and maximal consistency-preserving update operators, which additionally require the selected proxy to be maximally informative among the credible consequences of the original epistemic input. Finally, we establish the generality of the proposed framework by showing that credibility-limited belief update is recovered as a special case, while Katsuno--Mendelzon belief update emerges when credibility restrictions are removed and the transformation functions are taken to be identities. These results demonstrate that the framework provides a unified and strictly more expressive account of belief update, encompassing established approaches while supporting source-dependent selective acceptance.
\end{abstract}

\vspace{4mm}

\noindent\keywords{Belief Update, Non-Prioritized Belief Change, Selective Acceptance, Credibility-Limited Belief Update, Knowledge Representation and Reasoning}

\newpage

\section{Introduction}

Belief change concerns the rational modification of an agent's beliefs in response to new information \cite{ferme18}. A fundamental distinction is commonly drawn between \emph{belief revision}, which is appropriate when the new information corrects the agent's description of a fixed world \cite{agm85,gard88}, and \emph{belief update}, which is appropriate when the information reports a change in the world itself. Belief update was introduced in an early form by Keller and Winslett~\cite{keller85} and was subsequently given its standard axiomatic and semantic characterization by Katsuno and Mendelzon~\cite{katsuno92}. In the Katsuno--Mendelzon (KM) framework, each possible world (or simply world) compatible with the agent's prior beliefs is treated as a possible initial state, and is equipped with an ordering of potential successor worlds. Updating by an epistemic input then amounts to selecting, from each initial world, the most plausible successor worlds satisfying that epistemic input.

Standard KM belief update is a \emph{prioritized} form of belief change, as the epistemic input is required to be incorporated into the resulting state of belief. This unconditional success requirement reflects the assumption that the reported change has occurred and must therefore be accommodated, even when doing so requires departing substantially from the agent's previous beliefs. In many applications, however, an epistemic input may describe a transition that is physically impossible, causally inadmissible, unreliable, or otherwise unrealizable from some of the worlds regarded as initially possible. In such cases, unconditional priority is too demanding.

\emph{Non-prioritized} belief change relaxes the success requirement by allowing an epistemic input to be rejected or only partially accepted when it fails to meet the relevant admissibility conditions \cite{ferme26}. The resulting change operation is not required to accept every epistemic input in its entirety; instead, acceptance may depend on factors such as credibility, consistency, reliability, or reachability. In the update setting, this dependence is naturally {\em source-relative}, because an epistemic input may be realizable from one initially possible world but not from another. Therefore, a rational non-prioritized update mechanism must determine, for each source world, whether the reported transition can be accepted and how the corresponding branch should be treated when complete acceptance is not possible.

The \emph{credibility-limited} belief-update framework of Ferm{\'e} \emph{et al.} introduces source-relative credibility restrictions into belief update, by associating each source world with a restricted set of credible or reachable successors~\cite{ferme23}. Only credible worlds of the epistemic input may be selected. Under the credibility-limited (CL) approach, a source branch contributes no successor when the epistemic input has no credible world, relative to that source world; consequently, unconditional success is retained, but consistency may fail. Consistent credibility-limited (CCL) belief update adopts the complementary policy of retaining the source world whenever the epistemic input cannot be credibly realized. This preserves consistency, but permits the epistemic input to be locally rejected. Thus, both mechanisms depart from standard KM belief update by making the treatment of each source branch dependent on source-relative credibility, although only CCL permits local rejection of the epistemic input.

Although these approaches provide credibility-sensitive departures from standard KM belief update, they continue to treat the epistemic input as an {\em indivisible whole}. Relative to each source world, the complete epistemic input is either realized, rejected, or rendered ineffective through the elimination of the corresponding branch. This all-or-nothing treatment is restrictive when the epistemic input expresses a compound change whose components need not be equally realizable. For example, suppose that a robot is instructed to move a cup to a table and fill it. If the cup is intact, both parts of the instruction may be executable. If the cup is broken, however, the robot may still be able to move it, while being unable to fill it. In such a real-world scenario, standard KM belief update may represent the complete instruction as successful by selecting a physically impossible successor. Credibility-limited belief update may instead eliminate the broken-cup branch, while its consistent variant may reject the instruction altogether and leave the cup in its original position. None of these outcomes represents the {\em intermediate case} in which only the executable part of the compound instruction is realized.

The present article introduces \emph{selective credibility-limited belief update}, abbreviated as \emph{SCL belief update}, as a more flexible form of non-prioritized belief update. The framework builds on three principal lines of research; namely, the pointwise semantics of KM belief update~\cite{katsuno92}, the credibility-limited belief-update models of Ferm{\'e} \emph{et al.}~\cite{ferme23}, and the transformation-based methodology developed for selective belief revision~\cite{ferme99,garapa21}. Rather than requiring a source-relative choice between complete acceptance and complete rejection, SCL belief update permits the epistemic input to be transformed into a {\em weaker proxy} before the transition mechanism is applied. As a consequence, acceptance can vary not only between source worlds, but also in extent, as the same epistemic input may be fully accepted from one source world and selectively weakened from another. This makes it possible to represent the partial realization of compound epistemic inputs without either admitting non-credible successor worlds, eliminating the corresponding source branch, or leaving that branch entirely unchanged.

The contributions of the article are both foundational and comparative. First, we provide semantic and axiomatic characterizations of SCL update operators, in terms of credible faithful assignments~\cite{ferme23} and source-dependent transformation functions~\cite{ferme99}. Second, we identify two progressively more constrained sub-classes. {\em Consistency-preserving} SCL update operators require the transformed epistemic input to be credible from its corresponding source world whenever the original epistemic input is consistent, thus ensuring that every initially possible world contributes at least one successor in such cases. {\em Maximal consistency-preserving} SCL update operators additionally require, for every consistent epistemic input, the selected proxy to be maximally informative among its credible consequences. Third, we isolate an unrestricted-credibility variant that captures selective belief update {\em independently} of credibility limitations. Finally, we locate the established KM, CL, and CCL approaches within the proposed framework; accordingly, CL belief update is recovered by means of identity transformations, CCL belief update by transformations that retain the source world whenever the epistemic input is not locally credible, and KM belief update by additionally removing credibility restrictions. We further establish that the principal inclusion relations among these classes are proper, showing that SCL belief update constitutes a {\em strictly more expressive} framework than the approaches it subsumes.

\paragraph{Roadmap.}
The remainder of the article is organized as follows. Section~\ref{sec_related_work} situates the proposed framework within the literature on non-prioritized belief change. Section~\ref{sec_formal_prel} introduces the logical notation and preliminary results used throughout the article. Section~\ref{section_update} reviews the axiomatic and semantic foundations of KM belief update. Section~\ref{sec_credible_update} presents CL and CCL belief update, and illustrates their limitations through a real-world cup scenario. Section~\ref{sec_selective_credible_update} introduces SCL belief update and establishes its semantic and axiomatic characterization. Section~\ref{sec_classes_scl_operators} studies consistency-preserving and maximal consistency-preserving SCL update operators. Section~\ref{sec_cl_ccl_special_cases} shows how KM, CL, and CCL belief update arise within the SCL hierarchy, and establishes the strictness of the corresponding inclusions. The final section summarizes the main results and outlines directions for future research.

\section{Related Work}
\label{sec_related_work}

The proposed framework lies at the intersection of several lines of research on belief change and the partial realization of compound information. This section situates SCL belief update in relation to existing non-prioritized approaches to belief change and clarifies the respects in which it differs from them.

Research on non-prioritized belief change has concentrated primarily on belief revision, giving rise to models in which incoming information may be rejected, partially accepted, or assessed according to different levels of credibility; see the recent survey by Ferm{\'e}, Garapa, and Reis~\cite{ferme26}. Among these models, {\em selective revision} is particularly relevant to the present work, since it allows an epistemic input to be replaced by a weaker proxy representing the part of the information that the agent is prepared to accept~\cite{ferme99,garapa21}. Non-prioritized belief update has received comparatively less attention, although two distinct approaches have recently been developed within, or in close connection with, the KM framework. First, Grimaldi, Martinez, and Rodriguez introduce \emph{local promotion}, an update counterpart of belief promotion \cite{schwind18}, in which the new information need not receive unconditional priority~\cite{grimaldi21}. Local promotion is represented through two underlying update operators and a trigger function that determines which part of the prior information may be preserved relative to the epistemic input. Therefore, its principal concern is the controlled preservation of old information when the prior state and the epistemic input are combined. Second, the credibility-limited belief-update framework of Ferm{\'e} \emph{et al.} associates each source world with a restricted set of credible or reachable successor worlds~\cite{ferme23}. Under this approach, source-relative credibility determines whether the complete epistemic input has an admissible realization: CL eliminates a source branch when no such realization exists, whereas CCL retains the source world and thereby permits local rejection of the epistemic input.

A related, but conceptually different, treatment of compound epistemic inputs is provided by the \emph{compositional belief update} of Delgrande, Jin, and Pelletier~\cite{delgrande08}. Their operator recursively decomposes the syntactic structure of the update formula and combines the results obtained from its constituent sub-formulae. Hence, compositional belief update makes the internal structure of a compound epistemic input operationally significant and supports comparatively direct implementations. Nevertheless, it retains success for the complete epistemic input and, in its general form, does not satisfy the full set of standard KM postulates, including invariance under the substitution of logically equivalent update formulae. Therefore, its purpose is different from selective acceptance --- it determines how a compound epistemic input is realized from its syntactic composition, rather than whether only a weaker part of that epistemic input should be accepted.

The present article addresses a distinct problem by combining the pointwise credibility restrictions of credibility-limited belief update~\cite{ferme23}, with the transformation-based methodology of selective revision~\cite{ferme99}. In SCL belief update, the epistemic input is weakened semantically and independently at each source world before credible successor selection is performed. Unlike local promotion, SCL belief update does not obtain non-prioritization by combining separate contributions associated with the prior state and the new information; instead, it directly transforms the epistemic input into a source-relative proxy. Unlike compositional belief update, the transformation is invariant under logical equivalence and is not tied to the syntactic decomposition of the epistemic input. Consequently, SCL belief update can represent not only whether an epistemic input is accepted from a given source world, but also the extent to which it is accepted there, while excluding non-credible successor worlds. This idea also bears a conceptual affinity to partial-satisfaction planning, where an agent seeks valuable solutions when the complete set of objectives cannot be achieved, including recent extensions to hierarchical task-network planning~\cite{behnke23}. The difference is that SCL belief update operates at the epistemic level, selecting a credible logical consequence of the epistemic input relative to each source world, rather than optimizing over subsets of planning objectives.

\section{Formal Prelude}
\label{sec_formal_prel}

We assume a non-empty {\em finite} propositional signature $\mathcal{P}$, and let $\mathcal{L}$ be the propositional language constructed from the atoms in $\mathcal{P}$, using the standard Boolean connectives $\neg$, $\wedge$, $\vee$, $\rightarrow$, and $\leftrightarrow$. All formulae are evaluated under classical propositional semantics, and $\models$ denotes classical entailment. The symbol $\top$ denotes an arbitrary, but fixed, tautology.

A {\em possible world}, or simply a \emph{world}, is a truth assignment $w:\mathcal{P}\mapsto\{0,1\}$. We shall identify $w$ with the corresponding complete set of literals
\[
\big\{p\in\mathcal{P}:w(p)=1\big\} \ \cup\ \big\{\neg p:p\in\mathcal{P}\text{ and }w(p)=0\big\}.
\]
The collection of all possible worlds is denoted by $\mathbb{M}$. Since $\mathcal{P}$ is finite, so is $\mathbb{M}$. For readability, relative to a fixed ordering of the atoms, worlds will often be represented as sequences of literals, with set braces and commas omitted. Moreover, the negation of an atom $a\in\mathcal{P}$ may be written as $\bar a$; thus, for example, the world $\{a,\neg b,c\}$ may be displayed as $a\bar b c$.

For every sentence $\varphi\in\mathcal{L}$, we write $[\varphi]=\big\{w\in\mathbb{M}:w\models\varphi\big\}$ for the set of worlds satisfying $\varphi$. This notation extends to arbitrary sets of formulae $\Gamma\subseteq\mathcal{L}$ by \mbox{$[\Gamma] = \Big\{w\in\mathbb{M}:w\models\gamma\text{, for every }\gamma\in\Gamma\Big\}$}. Two sentences $\varphi,\psi\in\mathcal{L}$ are logically equivalent, written $\varphi\equiv\psi$, iff they have the same worlds; that is, $\varphi\equiv\psi$ iff $[\varphi]=[\psi]$.

For $\Gamma\subseteq\mathcal{L}$, its set of classical consequences is $Cn(\Gamma) = \big\{\varphi\in\mathcal{L}:\Gamma\models\varphi\big\}$. For a sentence $\varphi\in\mathcal{L}$, we use $Cn(\varphi)$ as an abbreviation for $Cn\big(\{\varphi\}\big)$. A set $K\subseteq\mathcal{L}$ is called a \emph{theory}, or \emph{belief set}, whenever it is closed under classical consequence, that is, whenever $K=Cn(K)$. A theory $K$ is \emph{complete} if $[K]$ is a singleton. The \emph{expansion} of a belief set $K$ by a formula $\varphi$ is denoted by $K+\varphi$, and is given by
\[
K+\varphi = Cn\bigl(K\cup\{\varphi\}\bigr).
\]
Intuitively, expansion incorporates $\varphi$ into $K$ without removing any of the beliefs already contained in $K$, and then closes the resulting set under logical consequence.

Let $X\subseteq\mathbb{M}$. A binary relation $\preceq$ over $X$ is a \emph{partial preorder}, or simply a \emph{preorder}, if it is reflexive and transitive. Its strict component $\prec$ is defined, for $r,r'\in X$, by $r\prec r'$ iff  $r\preceq r'$ and $r'\not\preceq r$. Given $Y\subseteq X$, the worlds in $Y$ that are minimal according to
$\preceq$ form the set
\[
\min(Y,\preceq) = \Big\{ r\in Y: \text{there exists no }r'\in Y\text{ such that }r'\prec r \Big\}.
\]

Finally, we record the following elementary property of minimal elements with respect to a partial preorder, which will be used in Subsection~\ref{subsec_axiomatic_char}.

\begin{lemma} \label{lem_mutual_minimality}
Let $\preceq$ be a partial preorder over $X\subseteq\mathbb{M}$, and let $A,B\subseteq X$. If $\min(A,\preceq)\subseteq B$ and $\min(B,\preceq)\subseteq A$, then $\min(A,\preceq)=\min(B,\preceq)$.
\end{lemma}

\begin{proof}
Suppose that $\min(A,\preceq)\subseteq B$ and $\min(B,\preceq)\subseteq A$.

We first show the inclusion $\min(A,\preceq)\subseteq\min(B,\preceq)$. Let $r\in\min(A,\preceq)$. By $\min(A,\preceq)\subseteq B$, we derive that $r\in B$. Suppose, towards a contradiction, that $r\notin\min(B,\preceq)$. Since $\mathbb{M}$ is finite, there exists $r'\in\min(B,\preceq)$ such that $r'\prec r$. By $\min(B,\preceq)\subseteq A$, we derive that $r'\in A$. Hence, $r'\in A$ and $r'\prec r$, contradicting $r\in\min(A,\preceq)$. Therefore, $r\in\min(B,\preceq)$. Thus, $\min(A,\preceq)\subseteq\min(B,\preceq)$.

The reverse inclusion, $\min(B,\preceq)\subseteq\min(A,\preceq)$, follows symmetrically. Consequently, $\min(A,\preceq)=\min(B,\preceq)$, as desired.
\end{proof}

\section{Belief Update}
\label{section_update}

Belief update is intended to model changes in an agent's beliefs when the new information reports an evolution of the underlying world, rather than, as in belief revision, indicating that the agent's previous description of an unchanged world was mistaken. Early foundations of this form of belief change were developed by Keller and Winslett~\cite{keller85}, while its standard formal treatment was subsequently provided by Katsuno and Mendelzon~\cite{katsuno92}. In this section, we briefly recall the axiomatic characterization of Katsuno--Mendelzon (KM) belief update and its corresponding semantic representation.

\subsection{Axiomatic Characterization}

An \emph{update operator} is a function $\diamond$ that maps each pair consisting of a theory $K$ and a formula $\varphi\in\mathcal{L}$ to a theory $K\diamond\varphi$. The formula $\varphi$ is referred to as the \emph{epistemic input}, while $K\diamond\varphi$ represents the agent's resulting state of belief after the change described by $\varphi$ is incorporated into $K$. The operator $\diamond$ is said to be a \emph{KM update operator} precisely when it satisfies the KM postulates listed below~\cite{katsuno92}.

\vspace{4mm}

{\renewcommand{\arraystretch}{1.5}
\begin{tabular}{l l}
$\bf (K\diamond1)$ & $K\diamond\varphi$ is a theory. \\

$\bf (K\diamond2)$ & $\varphi\in K\diamond\varphi$. \\

$\bf (K\diamond3)$ & If $\varphi\in K$, then $K\diamond\varphi = K$. \\

$\bf (K\diamond4)$ & If both $K$ and $\varphi$ are consistent, then $K\diamond\varphi$ is consistent as well. \\

$\bf (K\diamond5)$ & If $\varphi \equiv \psi$, then $K\diamond\varphi = K\diamond\psi$. \\

$\bf (K\diamond6)$ & $K\diamond(\varphi\wedge\psi) \subseteq (K\diamond\varphi)+\psi$.\\

$\bf (K\diamond7)$ & If $\psi\in K\diamond\varphi$ and $\varphi\in K\diamond\psi$, then $K\diamond\varphi = K\diamond\psi$. \\

$\bf (K\diamond8)$ & If $K$ is complete, then $K\diamond(\varphi\vee\psi) \subseteq Cn\big((K\diamond\varphi) \cup (K\diamond\psi) \big)$. \\

$\bf (K\diamond9)$ & If $[K]\neq\varnothing$, then $K\diamond\varphi\ = \bigcap\limits_{w\in[K]} \Big( Cn(w)\diamond\varphi \Big)$.
\end{tabular}}

\vspace{4mm}

Postulates $(K\diamond1)$--$(K\diamond9)$ restate the standard Katsuno--Mendelzon conditions (U1)--(U8) for the setting in which epistemic states are represented by deductively closed theories~\cite[p. 189]{katsuno92}. Since belief states are commonly represented in this way in the belief-change literature, we adopt the theory-based formulation throughout the article rather than the original formula-based presentation.

The role of each postulate may be summarized as follows. Postulate $(K\diamond1)$ ensures that every update result is closed under logical consequence, and $(K\diamond2)$ imposes success by requiring the epistemic input to be fully accepted. Postulate $(K\diamond3)$ states that no change occurs when the epistemic input is already believed, whereas $(K\diamond4)$ preserves consistency whenever both the prior theory and the epistemic input are consistent. Postulate $(K\diamond5)$ makes the operation insensitive to logically equivalent representations of the epistemic input. Postulate $(K\diamond6)$ constrains the relation between updating by a conjunction and subsequently expanding by one of its conjuncts. Postulate $(K\diamond7)$ identifies the outcomes of two updates when each epistemic input is accepted after updating by the other, and $(K\diamond8)$ governs the treatment of disjunctive epistemic inputs when the initial theory is complete. Finally, $(K\diamond9)$ captures the distinctive pointwise nature of belief update, by requiring the update of a consistent theory to be determined by the updates of the complete worlds compatible with it.

\subsection{Semantic Characterization}
\label{subsection_semantic_update}

Postulates $(K\diamond1)$--$(K\diamond9)$ have a pointwise semantic characterization based on preference relations over possible successor worlds. More precisely, each world $w$ that may describe the initial state is associated with its own partial preorder over $\mathbb{M}$. This relation ranks the possible outcomes of a change relative to $w$: if $r\prec_w r'$, then $r$ is regarded as a strictly more plausible successor of $w$ than $r'$.

\begin{definition}[Faithful Pointwise Assignment, {\cite{katsuno92}}]
A faithful pointwise assignment is a mapping that associates every world $w\in\mathbb{M}$ with a partial preorder $\preceq_w$ over $\mathbb{M}$ such that, for every $r\in\mathbb{M}$ distinct from $w$, $w\prec_w r$.
\end{definition}

Therefore, the faithfulness requirement makes the source world $w$ strictly preferred to every other world in its associated preorder. Using these source-relative plausibility relations, the following representation theorem characterizes KM update operators semantically.

\begin{theorem}[\cite{katsuno92}] \label{thm_update_char}
An update operator $\diamond$ satisfies postulates \mbox{$(K\diamond1)$--$(K\diamond9)$} iff there exists a faithful pointwise assignment that maps every world $w\in\mathbb{M}$ to a partial preorder $\preceq_w$ over $\mathbb{M}$, such that, for every belief set $K$ and every sentence $\varphi\in\mathcal{L}$,

\begin{center}
\begin{tabular}{l l}
{\bf (U)} & $[K\diamond\varphi]\ = \bigcup\limits_{w\in[K]} \min([\varphi],\preceq_w)$.
\end{tabular}
\end{center}
\end{theorem}

Condition (U) highlights the pointwise structure of KM belief update.\footnote{This pointwise semantics contrasts with the standard possible-worlds semantics for belief revision, in which a single plausibility ordering is associated with the prior theory as a whole~\cite{katsuno91}. The semantic relationship between revision and update, and the precise nature of their differences, have been further investigated in~\cite{peppas96a,aravanis25,bonanno25}.} To update a theory $K$ by an epistemic input $\varphi$, each world $w\in[K]$ is considered separately. From every such source world, the update mechanism selects the $\preceq_w$-minimal worlds satisfying $\varphi$. The worlds of the resulting theory $K\diamond\varphi$ are then obtained by collecting the worlds selected from all source worlds compatible with $K$.

\begin{remark}[Consistency Assumption]
Throughout the article, unless explicitly stated otherwise, we restrict attention to consistent prior belief sets. This restriction involves no loss of generality. Indeed, the only inconsistent belief set is $\mathcal L$, and postulate $(K\diamond3)$ entails that $\mathcal L\diamond\varphi=\mathcal L$, for every epistemic input $\varphi$. Correspondingly, in each semantic representation considered herein, $[\mathcal L]=\varnothing$, so the union indexed by the worlds of the prior belief set is empty, and therefore represents the inconsistent belief set $\mathcal L$. Thus, all representation results extend immediately to the inconsistent prior case.
\end{remark}

The following running example, adapted from the cup-domain scenario of Ferm{\'e} \emph{et al.}~\cite{ferme23}, will be used throughout the article to illustrate the behaviour and limitations of the update mechanisms considered. We first examine the scenario under standard KM belief update, and subsequently revisit it in the context of credibility-limited and selective forms of belief update.

\begin{example}[Running Cup Scenario] \label{ex_running_cup} Let \(t\) denote that the cup is on the table, \(e\) that the cup is empty, and \(b\) that the cup is broken. The agent's initial state of belief is represented by 
\[ 
K=Cn\Bigl((t\land\neg e\land\neg b)\lor(\neg t\land e\land\neg b)\lor(\neg t\land e\land b)\Bigr). 
\] 
Thus, \([K]=\{w_1,w_2,w_3\}\), where $w_1=t\bar e\bar b$, $w_2=\bar t e\bar b$, and $w_3=\bar t e b$. According to \(w_1\), the cup is on the table, non-empty, and intact; according to \(w_2\), it is on the floor, empty, and intact; and according to \(w_3\), it is on the floor, empty, and broken. 

Assume that a broken cup is necessarily empty and that brokenness is irreversible. These physical constraints will later be encoded by assigning a credible set \(C_w\) of worlds to each source world \(w\in\{w_1,w_2,w_3\}\) (see Definition~\ref{def_cred_assign}):
\[ 
C_{w_1}=C_{w_2}=[b\rightarrow e]=\Big\{teb,te\bar b,t\bar e\bar b,\bar t eb,\bar t e\bar b,\bar t\bar e\bar b\Big\} \qquad \text{and} \qquad C_{w_3}=[b\land e]=\{teb,\bar t eb\}. 
\] 
A robot is instructed to move the cup to the table and fill it. The intended postcondition is 
\[ 
\varphi = t\land\neg e, \qquad \text{with} \qquad [\varphi]=\{t\bar e\bar b,t\bar e b\}.
\]
The complete postcondition can be realized when the cup is intact. If the cup is broken, however, only its table-location component can be realized.

Now, consider a KM update operator $\diamond$ whose pointwise preorders select \[ \min([\varphi],\preceq_{w_1})=\min([\varphi],\preceq_{w_2})=\{t\bar e\bar b\} \qquad \text{and} \qquad \min([\varphi],\preceq_{w_3})=\{t\bar e b\}. \] Then, condition (U) yields $[K\diamond_{\mathrm{KM}}\varphi]=\{t\bar e\bar b,t\bar e b\}$, and therefore, \[ K\diamond_{\mathrm{KM}}\varphi=Cn(t\land\neg e). \] The world \(t\bar e b \in [K\diamond_{\mathrm{KM}}\varphi]\) represents a cup that is both broken and non-empty. Thus, standard KM belief update guarantees acceptance of the complete postcondition, but does so at the cost of violating the physical constraint that every broken cup is empty. 
\end{example}

\section{Credibility-Limited Belief Update}
\label{sec_credible_update}

Ferm{\'e} \emph{et al.}~\cite{ferme23} modify the pointwise semantics of KM belief update by restricting, for each source world, the successor worlds that are regarded as credible or reachable. The resulting framework weakens standard KM belief update by restricting the admissible successors of each source world. {\em Credibility-limited} (CL) belief update retains formal success but may sacrifice consistency, whereas {\em consistent credibility-limited} (CCL) belief update preserves consistency by permitting the epistemic input to be rejected locally. Both approaches are based on the following semantic structure.

\begin{definition}[Credible Faithful Assignment, \cite{ferme23}] \label{def_cred_assign}
A credible faithful assignment is a mapping that associates every world \(w \in \mathbb{M}\) with a pair $\left(C_{w},\preceq_{w}\right)$, where $\{w\} \subseteq C_{w} \subseteq \mathbb{M}$, and \(\preceq_{w}\) is a partial preorder over \(C_{w}\) such that, for every \(w' \in C_{w}\), if \(w' \neq w\), then $w \prec_{w} w'$.
\end{definition}

The non-empty set $C_w$ contains the worlds that are regarded as credible successors of $w$; therefore, worlds outside $C_w$ are excluded as possible outcomes of a transition from $w$. The preorder $\preceq_w$ ranks the worlds in $C_w$ according to their comparative transition plausibility. Faithfulness ensures that $w$ is the uniquely most plausible element of its own credible set.

\subsection{The CL Approach}
\label{subsec_cl_update}

CL belief update retains all the standard KM postulates $(K\diamond1)$--$(K\diamond9)$, except consistency preservation, encoded in postulate $(K\diamond4)$. Therefore, it requires the epistemic input to be accepted, but allows the updated belief set to become inconsistent when the epistemic input is not credibly realizable from the prior state of belief.

\begin{definition}[CL Update Operator, \cite{ferme23}] \label{def_cl_operator}
An update operator $\diamond$ is a CL update operator iff it satisfies postulates $(K\diamond1)$--$(K\diamond3)$ and $(K\diamond5)$--$(K\diamond9)$.
\end{definition}

The corresponding representation theorem restricts the KM-style minimization process to the credible successors of each source world.

\begin{theorem}[\cite{ferme23}] \label{thm_cl_charact}
An update operator $\diamond$ is a CL update operator iff there exists a credible faithful assignment $w \mapsto \left(C_{w},\preceq_{w}\right)$, such that, for every belief set $K$ and every sentence $\varphi\in\mathcal{L}$,

\begin{center}
\begin{tabular}{l l}
{\bf (CL)} & $[K\diamond\varphi] = \displaystyle\bigcup_{w\in[K]} \min\Big([\varphi]\cap C_{w}\,,\preceq_w\!\Big)$.
\end{tabular}
\end{center}
\end{theorem}

Thus, each source world $w\in[K]$ contributes its most plausible credible $\varphi$-successors. If $[\varphi]\cap C_w=\varnothing$, the branch originating from $w$ contributes no successor. In particular, if this occurs for every $w\in[K]$, then the update result is inconsistent.

The following continuation of Example~\ref{ex_running_cup} of the previous section illustrates how this local elimination of non-credible branches affects the update result. In particular, CL belief update avoids the physically impossible successor produced by standard KM belief update, but may do so by discarding an initially possible source branch altogether.

\begin{example}[Cup Scenario under CL Belief Update] \label{ex_cup_cl}
Let $\diamond_{\mathrm{CL}}$ be a CL update operator. Recall that the physical constraints of the scenario are encoded by the credible sets
\[ 
C_{w_1}=C_{w_2}=[b\rightarrow e] \qquad \text{and} \qquad C_{w_3}=[b\land e]. 
\] 
Hence, for the source worlds \(w_1\) and \(w_2\), the epistemic input \(\varphi = t\land\neg e\) has credible worlds, and 
\[ 
\min\Big([\varphi]\cap C_{w_1},\preceq_{w_1}\!\Big)=\min\Big([\varphi]\cap C_{w_2},\preceq_{w_2}\!\Big)=\{t\bar e\bar b\}. 
\] 
For the broken-cup world \(w_3\), however, $[\varphi]\cap C_{w_3}=\varnothing$. Consequently, we have from condition (CL) that $[K\diamond_{\mathrm{CL}}\varphi]=\{t\bar e\bar b\}$, and hence, 
\[ 
K\diamond_{\mathrm{CL}}\varphi=Cn\big(t\land\neg e\land\neg b\big). 
\] 
Observe that CL belief update avoids the physically impossible successor \(t\bar e b\). However, the branch originating from \(w_3\) contributes no successor at all. That is to say, the result fails to represent the fact that, although the broken cup cannot be filled, it can still be moved to the table. The elimination of the entire source branch motivates the consistent variant considered in the next subsection, which preserves the branch but, as will be shown, still does not capture this partial realization.
\end{example}

\subsection{The CCL Approach}
\label{subsec_ccl_update}

CCL belief update adopts a different policy for locally non-credible epistemic inputs. Rather than eliminating a source branch, it leaves the corresponding source world unchanged. This restores consistency, but requires unconditional success to be replaced by a weaker, source-relative acceptance principle.

\begin{definition}[CCL Update Operator, \cite{ferme23}] \label{def_ccl_operator}
An update operator $\diamond$ is a CCL update operator iff it satisfies postulates $(K\diamond1)$, $(K\diamond3)$--$(K\diamond9)$, together with the following postulates:

\begin{center}
{\renewcommand{\arraystretch}{1.5}
\noindent \begin{tabular}{l l}
{\bf (RSC)} & If $K$ is complete, then $\varphi\in K\diamond\varphi$ or $K\diamond\varphi = K$. \\

{\bf (SM)} & If $\varphi\models\psi$ and $\varphi\in K\diamond\varphi$, then $\psi\in K\diamond\psi$. \\

{\bf (IR)} & If $\varphi$ is inconsistent, then $K\diamond\varphi=K$.
\end{tabular}}
\end{center}
\end{definition}

Postulate (RSC), relative success for complete theories, requires an update originating from a single source world either to accept the epistemic input or to leave that source world unchanged. Postulate (SM), success monotonicity, ensures that whenever an epistemic input is successfully accepted, its logical consequences are successfully accepted as well. Finally, postulate (IR) requires every inconsistent epistemic input to be rejected.\footnote{Postulate (IR) is not stated explicitly in the axiomatization of CCL belief update presented in~\cite{ferme23}. However, the proof of the corresponding representation theorem implicitly restricts attention to consistent epistemic inputs. We include (IR) herein to obtain the stated characterization over arbitrary epistemic inputs.}

The following representation theorem provides the corresponding semantic characterization of CCL belief update.

\begin{theorem}[\cite{ferme23}] \label{thm_ccl_charact}
An update operator $\diamond$ is a CCL update operator iff there exists a credible faithful assignment $w \mapsto \left(C_{w},\preceq_{w}\right)$, such that, for every belief set $K$ and every sentence $\varphi\in\mathcal{L}$,

\begin{center}
\begin{tabular}{l l}
{\bf (CCL)} & $[K\diamond\varphi] = \displaystyle\bigcup_{w\in[K]} g_{w}(\varphi)$,
\end{tabular}
\end{center}

\noindent where, for every $w\in\mathbb{M}$ and every $\varphi\in\mathcal{L}$,
\[
g_{w}(\varphi) =
\begin{cases}
\min\Big([\varphi]\cap C_{w}\, , \preceq_{w}\!\Big),
&
\text{if }\ [\varphi]\cap C_{w} \neq\varnothing,
\\[1mm]
\{w\},
&
\text{otherwise.}
\end{cases}
\]
\end{theorem}

Consequently, CL and CCL belief update coincide at every source world \(w\) for which the epistemic input has a credible world. They differ only when \([\varphi]\cap C_w=\varnothing\); specifically, CL belief update discards the branch originating from \(w\), whereas CCL belief update retains \(w\), thus, representing the failure of the transition without eliminating the corresponding initial possibility. The running cup scenario illustrates this distinction.

\begin{example}[Cup Scenario under CCL Belief Update] \label{ex_cup_ccl}
Let \(\diamond_{\mathrm{CCL}}\) be a CCL update operator represented by the credible faithful assignment considered in Example~\ref{ex_cup_cl}. For the source worlds \(w_1\) and \(w_2\), CCL belief update behaves exactly as CL belief update, since \(\varphi = t\land\neg e\) has credible worlds relative to both worlds. For the broken-cup world \(w_3\), however, no credible world of \(\varphi\) is reachable. Hence, we derive that
\[ 
g_{w_3}(\varphi)=\{w_3\}=\{\bar t e b\}. 
\]
Then, it follows from condition (CCL) that $[K\diamond_{\mathrm{CCL}}\varphi]=\{t\bar e\bar b,\bar t e b\}$, and hence, 
\[ 
K\diamond_{\mathrm{CCL}}\varphi=Cn\Bigl((t\land\neg e\land\neg b)\lor(\neg t\land e\land b)\Bigr). 
\]
Thus, CCL belief update preserves the broken-cup branch but leaves it unchanged, thereby failing to realize any part of the compound postcondition on that branch. Consequently, the broken cup remains on the floor, even though moving it to the table is physically possible.
\end{example}

The running example exposes a limitation shared by CL and CCL belief update. Both treat the compound epistemic input \(\varphi = t\land\neg e\) as an {\em indivisible whole}. When the complete epistemic input is not credible, CL belief update eliminates the corresponding source branch, whereas CCL belief update rejects the complete epistemic input and retains the source world. Neither mechanism captures the intermediate case in which only one component of the compound postcondition is realized. The {\em selective} framework introduced in the next section relaxes this all-or-nothing treatment.

\section{Selective Acceptance in Credibility-Limited Belief Update}
\label{sec_selective_credible_update}

Having reviewed standard KM belief update and its credibility-limited generalizations, we now introduce {\em selective acceptance} into the update setting. As shown, the existing credibility-limited approach by Ferm{\'e} \emph{et al.}~\cite{ferme23} continue to treat the epistemic input as an indivisible whole: relative to a source world, it is either fully realized, locally rejected by retaining the source world, or rendered ineffective by eliminating the source branch. This all-or-nothing treatment is inadequate for compound epistemic inputs whose components may differ in their credibility or executability.

The basic idea is to interpose a transformation stage before the underlying belief-change mechanism is applied. In this respect, the proposed construction is the belief-update analogue of {\em selective belief revision} \cite{ferme99}, in which an epistemic input is transformed into an acceptable proxy before being processed by a revision operator. In the update setting, however, the pointwise character of the KM semantics requires the accepted portion of the epistemic input to be determined relative to each source possible world. The semantic realization of this idea is developed below.

\subsection{Semantic Characterization}
\label{subsec_semantic_char}

Selective acceptance is represented by transforming the epistemic input before the credibility-limited transition is performed. Accordingly, we associate each source world \(w\) with a {\em transformation function} \(f_w\). Given an epistemic input \(\varphi\), the formula \(f_w(\varphi)\) represents the part of \(\varphi\) that is accepted from \(w\). The transition mechanism is then applied to this transformed epistemic input rather than directly to \(\varphi\). This source dependence preserves the pointwise semantics of KM belief update, while allowing the extent of acceptance to vary across source worlds. Thus, the same epistemic input may be fully accepted from one source world, and weakened from another. The collection of source-dependent transformations is formalized by the following assignment.

\begin{definition}[Pointwise Transformation Assignment] \label{def_tranf_assign}
A pointwise transformation assignment is a family $\mathcal{F} = \left\{ f_{w} : w \in \mathbb{M} \right\}$, where every $f_{w}:\mathcal{L}\mapsto\mathcal{L}$ is a transformation function.
\end{definition}

For every source world \(w\) and epistemic input \(\varphi\), the formula \(f_w(\varphi)\) is intended to represent the information conveyed by \(\varphi\) that can be accepted from \(w\). To capture this interpretation, a pointwise transformation assignment may satisfy the following basic properties.

\vspace{4mm}

{\renewcommand{\arraystretch}{1.5}
\begin{tabular}{l p{13cm}}
{\bf (F1)} & $\varphi \models f_{w}(\varphi)$. \\

{\bf (F2)} & If $\varphi \equiv \psi$, then $f_{w}(\varphi) \equiv f_{w}(\psi)$. \\

{\bf (F3)} & $f_{w}\big(f_{w}(\varphi)\big) \equiv f_{w}(\varphi)$. \\

{\bf (F4)} & If $[\varphi]\cap C_{w}\neq\varnothing$, then $f_{w}(\varphi) \equiv\varphi$.
\end{tabular}} 

\vspace{5mm}

Properties (F1)--(F3) are adapted from corresponding conditions on transformation functions in selective belief revision~\cite[p. 335]{ferme99}. More specifically, they are the source-dependent counterparts of {\em implication}, {\em extensionality}, and {\em idempotence}, respectively. Property (F1) permits the transformation to discard information, but prevents it from introducing information not implied by the original epistemic input. Property (F2) guarantees that the transformation is insensitive to syntactic representation, while property (F3) requires the selected proxy to be a fixed point of the transformation.

Property (F4) may be viewed as the source-relative credibility counterpart of {\em weak maximality} in selective belief revision~\cite[p. 335]{ferme99}. There, weak maximality requires the epistemic input to be retained unchanged whenever it is consistent with the prior belief set. Here, consistency with the prior belief set is replaced by local credibility from \(w\) --- whenever \([\varphi]\cap C_w\neq\varnothing\), no weakening occurs and the epistemic input is accepted unchanged from \(w\).

A pointwise transformation assignment $\mathcal{F}$ and a credible faithful assignment $w \mapsto \left(C_{w},\preceq_{w}\right)$ jointly determine the update process. For each source world \(w\), the epistemic input \(\varphi\) is first transformed into the proxy \(f_w(\varphi)\); the most plausible worlds in \(C_w\) satisfying that proxy are then selected according to the preorder \(\preceq_w\). Taking the union of these locally selected successors yields the updated belief set. This two-stage construction defines a \emph{selective credibility-limited update operator}, abbreviated as an \emph{SCL update operator}.

\begin{definition}[SCL Update Operator] \label{def_scl_operator}
An update operator $\diamond$ is an SCL update operator iff there exist a credible faithful assignment $w\mapsto(C_w,\preceq_w)$ and a pointwise transformation assignment $\mathcal F=\{f_w:w\in\mathbb M\}$ satisfying properties (F1)--(F4), such that, for every belief set $K$ and every sentence $\varphi\in\mathcal L$,

\begin{center}
\begin{tabular}{l l}
{\bf (SCL)} & $[K\diamond\varphi] = \displaystyle\bigcup_{w\in[K]}\min\Big([f_w(\varphi)]\cap C_w\,,\preceq_w\!\Big)$.
\end{tabular}
\end{center}
\end{definition}

The following continuation of the running cup scenario illustrates how source-dependent transformations allow a compound epistemic input to be only partially realized, thus, demonstrating the greater expressive capacity of SCL belief update relative to KM, CL, and CCL update.

\begin{example}[Cup Scenario under SCL Belief Update] \label{ex_cup_scl}
Let \(\diamond_{\mathrm{SCL}}\) be an SCL update operator represented by the credible faithful assignment considered in Examples~\ref{ex_cup_cl} and~\ref{ex_cup_ccl}. Regard the epistemic input $\varphi = t\land\neg e$ as a compound postcondition and consider the candidate proxies
\[ 
\top,\qquad t,\qquad\neg e,\qquad t\land\neg e. 
\] 
For the source worlds \(w_1\) and \(w_2\), the complete epistemic input is credible. Therefore, property (F4) gives 
\[ 
f_{w_1}(\varphi)\equiv\varphi\qquad\text{and}\qquad f_{w_2}(\varphi)\equiv\varphi. 
\] 
For the broken-cup world \(w_3\), neither the complete epistemic input nor the component \(\neg e\) is credible, since $[t\land\neg e]\cap C_{w_3}=\varnothing$ and $[\neg e]\cap C_{w_3}=\varnothing$, respectively. The component \(t\), however, is credible, as \mbox{$[t]\cap C_{w_3}=\{t e b\}$}. Accordingly, take the pointwise transformation assignment \(\mathcal F\) representing \(\diamond_{\mathrm{SCL}}\) to satisfy $f_{w_3}(\varphi)\equiv t$. Since $\min\Big([t]\cap C_{w_3},\preceq_{w_3}\!\Big)=\{t e b\}$, condition (SCL) gives $[K\diamond_{\mathrm{SCL}}\varphi]=\{t\bar e\bar b,t e b\}$. Equivalently, 
\[ 
K\diamond_{\mathrm{SCL}}\varphi=Cn\Bigl((t\land\neg e\land\neg b)\lor(t\land e\land b)\Bigr). 
\] 
Thus, from an intact-cup source world, the complete postcondition is realized. From the broken-cup source world, only its table-location component is realized. Consequently, SCL belief update represents the partial realization of the compound postcondition.\footnote{This partially realized outcome also clarifies the distinction between SCL belief update and both local promotion~\cite{grimaldi21} and compositional belief update~\cite{delgrande08}. Local promotion regulates how much of the prior information may be preserved while the new information is processed, whereas compositional belief update exploits the syntactic structure of the epistemic input in order to determine how that epistemic input is realized. In their standard formulations, however, neither approach transforms the epistemic input into a weaker proxy relative to the particular source world from which the complete transition is unrealizable. Consequently, neither directly represents the intermediate outcome in which only the realizable part of the compound epistemic input is effected.}

The different treatments of the broken-cup branch are summarized below: 
\[
{\renewcommand{\arraystretch}{1.5}
\begin{array}{c|c|c}
\text{\bf Update Mechanism} & \text{\bf Successor from \(w_3\)} & \text{\bf Interpretation}\\ \hline\hline 
\mathrm{KM} & t\bar e b & \text{Physically impossible full success}\\ 
\mathrm{CL} & \varnothing & \text{Source branch eliminated}\\ 
\mathrm{CCL} & \bar t e b & \text{Compound postcondition not realized}\\ 
\mathrm{SCL} & t e b & \text{Compound postcondition partially realized}
\end{array}}
\]

\vspace{1mm}

As the table shows, SCL belief update is the only update mechanism that both preserves the broken-cup branch and respects the physical constraints, while capturing the partial realization of the compound postcondition.
\end{example}

\subsection{Axiomatic Characterization}
\label{subsec_axiomatic_char}

Let us now proceed to the axiomatic characterization of SCL belief update. To that end, we retain postulates \((K\diamond1)\), \((K\diamond3)\), \((K\diamond5)\), \((K\diamond7)\), and \((K\diamond9)\) from the standard KM framework, and replace postulates \((K\diamond2)\), \((K\diamond6)\), and \((K\diamond8)\) with the following selective variants. Postulates \((K\diamond2)^S\), \((K\diamond6)^S\), and \((K\diamond8)^S\) are {\em weakenings} of their respective standard counterparts, reflecting the fact that an epistemic input need not be accepted in its entirety. We also introduce the local credibility postulate (LC), which connects the observable acceptance behaviour of the operator with source-relative credibility.

\vspace{4mm}

{\renewcommand{\arraystretch}{1.5}
\begin{tabular}{l p{13cm}}
$\mathbf{(K\diamond2)^{S}}$ & If $K$ is complete, then there is a \(\psi \in \mathcal{L}\) such that $\varphi \models \psi$, $\psi\in K\diamond\varphi$, and\newline $K\diamond\varphi = K\diamond\psi$. \\

$\mathbf{(K\diamond6)^{S}}$ & If $\varphi\in K\diamond\varphi$, then $K\diamond(\varphi\wedge\psi) \subseteq (K\diamond\varphi)+\psi$. \\

$\mathbf{(K\diamond8)^{S}}$ & If $K$ is complete, $\varphi\in K\diamond\varphi$, and $\psi\in K\diamond\psi$, then \newline $K\diamond(\varphi\vee\psi) \subseteq Cn\big((K\diamond\varphi)\cup(K\diamond\psi)\big)$. \\

{\bf (LC)} & If $K$ is complete, $\varphi\in K\diamond\varphi$, and $\neg\psi\notin K\diamond\varphi$, then $\psi\in K\diamond\psi$.
\end{tabular}}

\vspace{4mm}

Some comments on these postulates are in order. Postulate \((K\diamond2)^S\) is a selective counterpart of the KM success postulate \((K\diamond2)\). Rather than requiring the whole epistemic input \(\varphi\) to be accepted, it requires the existence of an accepted proxy \(\psi\) that is implied by \(\varphi\) and produces the same update result. In this respect, it is the source-local update analogue of the {\em proxy-success} postulate used in selective belief revision~\cite[p. 334]{ferme99}. The restriction to complete prior theories serves to isolate a single source world; when \(K=Cn(w)\), the update originates from a single source world \(w\), and the witnessing proxy can be identified with the corresponding transformed epistemic input \(f_w(\varphi)\).

Postulate \((K\diamond6)^S\) preserves the KM-style interaction between update and conjunction whenever the first conjunct is itself accepted. Without the condition \(\varphi\in K\diamond\varphi\), the epistemic input \(\varphi\) may have been replaced by a weaker proxy, and the standard conjunctive constraint would no longer be justified. Thus, the postulate reinstates the usual KM-style behaviour precisely in those cases in which selective weakening has not prevented the acceptance of \(\varphi\).

Similarly, postulate \((K\diamond8)^S\) retains the KM-style constraint on update by a disjunction only when each disjunct is accepted under its corresponding update. If either disjunct is selectively weakened, the updates by \(\varphi\), by \(\psi\), and by \(\varphi\vee\psi\) may be based on different proxies, so the unconditional KM disjunctive postulate $(K\diamond8)$ would be too strong.

Finally, postulate (LC) provides a behavioural expression of local credibility. Suppose that \(K=Cn(w)\) (for a world $w$), that the update by \(\varphi\) accepts \(\varphi\), and that \(\psi\) remains possible after that update. Then, some selected successor of \(w\) satisfies \(\psi\), showing that \(\psi\) is credible from \(w\). Condition (LC) requires this local credibility to be reflected in the result of updating directly by \(\psi\), which must then accept \(\psi\).

Against this background, we can state the following representation theorem for SCL belief update.

\begin{theorem} \label{thm_scl_charact}
An update operator $\diamond$ satisfies postulates $(K\diamond1)$, $(K\diamond2)^{S}$, $(K\diamond3)$, $(K\diamond5)$, $(K\diamond6)^{S}$, $(K\diamond7)$, $(K\diamond8)^{S}$, $(K\diamond9)$, and (LC) iff there exist a credible faithful assignment $w\mapsto(C_w,\preceq_w)$ and a pointwise transformation assignment $\mathcal F=\{f_w:w\in\mathbb M\}$ satisfying properties (F1)--(F4), such that, for every belief set $K$ and every sentence $\varphi\in\mathcal L$,
\begin{center}
\begin{tabular}{l l}
{\bf (SCL)} & $[K\diamond\varphi] = \displaystyle\bigcup_{w\in[K]}\min\Big([f_w(\varphi)]\cap C_w\,,\preceq_w\!\Big)$.
\end{tabular}
\end{center}
\end{theorem}

\begin{proof} \underline{\bf Right-to-left implication.} Let $\diamond$ be an update operator, let $K$ be a belief set, let $\varphi$ be a sentence of $\mathcal{L}$, and assume that there exist a credible faithful assignment \mbox{$w\mapsto(C_w,\preceq_w)$} and a pointwise transformation assignment $\mathcal F$ satisfying properties (F1)--(F4), such that condition (SCL) holds; that is, \mbox{$[K\diamond\varphi] = \displaystyle\bigcup_{w\in[K]}\min\Big([f_w(\varphi)]\cap C_w\,,\preceq_w\!\Big)$}. We show that $\diamond$ satisfies postulates \mbox{$(K\diamond1)$}, $(K\diamond2)^{S}$, $(K\diamond3)$, $(K\diamond5)$, $(K\diamond6)^{S}$, $(K\diamond7)$, $(K\diamond8)^{S}$, $(K\diamond9)$, and (LC).

\begin{itemize}
\item Postulate $(K\diamond1)$ follows immediately, since the
right-hand side of (SCL) determines a theory.

\item For postulate $(K\diamond2)^{S}$, let $K=Cn(w)$ be complete, let $\varphi$ be an epistemic input, and put \mbox{$\psi=f_w(\varphi)$}. By property (F1), $\varphi\models\psi$. Since $K$ is complete, we have that $[K] = \{w\}$. Hence, by condition (SCL), we have that $[K\diamond\varphi] = \min\Big([f_w(\varphi)]\cap C_w,\preceq_w\!\Big) \subseteq[\psi]$, so $\psi\in K\diamond\varphi$. Finally, by property (F3), we have that $f_w(\psi) = f_w\big(f_w(\varphi)\big) \equiv f_w(\varphi)$, and therefore, $K\diamond\varphi = K\diamond\psi$, as desired.

\item For postulate $(K\diamond3)$, suppose that $\varphi\in K$. Then, for every $w\in[K]$, we have that \mbox{$w\in[\varphi]\cap C_w$}. By property (F4), $f_w(\varphi)\equiv\varphi$. Since the preorder $\preceq_w$ is faithful to $w$, it is true that \mbox{$\min\Big([\varphi]\cap C_w,\preceq_w\!\Big)=\{w\}$}. Consequently, $[K\diamond\varphi] = \displaystyle\bigcup_{w\in[K]}\{w\} = [K]$, and hence, $K\diamond\varphi=K$.

\item Postulate $(K\diamond5)$ follows from property (F2), combined with condition (SCL).

\item For postulate $(K\diamond6)^{S}$, suppose that $\varphi\in K\diamond\varphi$. Let $\psi\in\mathcal{L}$ and fix an arbitrary world $w\in[K]$. We prove the following local inclusion:
\[
\min\Big([f_w(\varphi)]\cap C_w,\preceq_w\!\Big)\cap[\psi] \subseteq \min\Big(\big[f_w(\varphi\wedge\psi)\big]\cap C_w,\preceq_w\!\Big). \tag{1}
\]

\begin{itemize}
\item First, assume that $\min\Big([f_w(\varphi)]\cap C_w,\preceq_w\!\Big) = \varnothing$. Then, \mbox{$\min\Big([f_w(\varphi)]\cap C_w,\preceq_w\!\Big)\cap[\psi] = \varnothing$}. Hence, inclusion (1) follows trivially.

\item Next, assume that $\min\Big([f_w(\varphi)]\cap C_w,\preceq_w\!\Big) \neq \varnothing$. By condition (SCL), \mbox{$\min\Big([f_w(\varphi)]\cap C_w,\preceq_w\!\Big) \subseteq [K\diamond\varphi]$}. Since $\varphi\in K\diamond\varphi$, we have that $[K\diamond\varphi]\subseteq[\varphi]$. Therefore, $\min\Big([f_w(\varphi)]\cap C_w,\preceq_w\!\Big) \subseteq [\varphi]$. Moreover, $\min\Big([f_w(\varphi)]\cap C_w,\preceq_w\!\Big) \subseteq C_w$. Since $\min\Big([f_w(\varphi)]\cap C_w,\preceq_w\!\Big) \neq \varnothing$ by our standing assumption, there exists a world $r\in \min\Big([f_w(\varphi)]\cap C_w,\preceq_w\!\Big)$. Consequently, $r\in[\varphi]\cap C_w$, and hence, $[\varphi]\cap C_w\neq\varnothing$. Property (F4) therefore gives $f_w(\varphi)\equiv\varphi$, meaning that
\[
\min\Big([f_w(\varphi)]\cap C_w,\preceq_w\!\Big) = \min\Big([\varphi]\cap C_w,\preceq_w\!\Big). \tag{2}
\]

Now, if $[\varphi\wedge\psi]\cap C_w\neq\varnothing$, then property (F4) also gives 
\[
f_w(\varphi\wedge\psi)\equiv\varphi\wedge\psi. \tag{3}
\]
We show that $\min\Big([\varphi]\cap C_w,\preceq_w\!\Big)\cap[\psi] \subseteq \min\Big([\varphi\wedge\psi]\cap C_w,\preceq_w\!\Big)$. Let \mbox{$r\in\min\Big([\varphi]\cap C_w,\preceq_w\!\Big)\cap[\psi]$}. Then, $r\in[\varphi]\cap C_w$ and $r\in[\psi]$, so $r\in[\varphi\wedge\psi]\cap C_w$. Suppose, towards a contradiction, that $r\notin\min\Big([\varphi\wedge\psi]\cap C_w,\preceq_w\!\Big)$. Then, there exists $r'\in[\varphi\wedge\psi]\cap C_w$ such that $r'\prec_w r$. Since $[\varphi\wedge\psi]\cap C_w \subseteq [\varphi]\cap C_w$, we have that $r'\in[\varphi]\cap C_w$, contradicting $r\in\min\Big([\varphi]\cap C_w,\preceq_w\!\Big)$. Therefore, 
\[
\min\Big([\varphi]\cap C_w,\preceq_w\!\Big)\cap[\psi] \subseteq \min\Big([\varphi\wedge\psi]\cap C_w,\preceq_w\!\Big). \tag{4}
\] 
By conditions (2)--(4), we derive that inclusion (1) holds.

Suppose instead that $[\varphi\wedge\psi]\cap C_w=\varnothing$. Then, by condition (2),
\[
\begin{aligned}
\min\Big([f_w(\varphi)]\cap C_w,\preceq_w\!\Big)\cap[\psi]
&=
\min\Big([\varphi]\cap C_w,\preceq_w\!\Big)\cap[\psi]
\\
&\subseteq
[\varphi]\cap C_w\cap[\psi]
\\[1mm]
&=[\varphi\wedge\psi]\cap C_w = \varnothing.
\end{aligned}
\]
Therefore, \mbox{$\min\Big([f_w(\varphi)]\cap C_w,\preceq_w\!\Big)\cap[\psi] = \varnothing$}, which means that inclusion~(1) holds.
\end{itemize}

Since $w\in[K]$ was arbitrary, inclusion (1) holds for every $w\in[K]$. Taking the union over all $w\in[K]$, we obtain that
\[
\begin{aligned}
\big[(K\diamond\varphi)+\psi\big]
&=
[K\diamond\varphi]\cap[\psi]
\\
&=
\Bigg( \bigcup_{w\in[K]} \min\Big([f_w(\varphi)]\cap C_w\,,\preceq_w\!\Big) \Bigg)\cap[\psi]
\\
&=
\bigcup_{w\in[K]}
\Bigg( \min\Big([f_w(\varphi)]\cap C_w\,,\preceq_w\!\Big) \cap[\psi] \Bigg)
\\
&\subseteq
\bigcup_{w\in[K]}\min\bigg(\big[f_w(\varphi\wedge\psi)\big]\cap C_w\,,\preceq_w\!\bigg)
\\
&=
\big[K\diamond(\varphi\wedge\psi)\big].
\end{aligned}
\]
Therefore, $K\diamond(\varphi\wedge\psi) \subseteq (K\diamond\varphi)+\psi$, meaning that postulate $(K\diamond6)^{S}$ is satisfied.

\item For postulate $(K\diamond7)$, suppose that $\psi\in K\diamond\varphi$ and $\varphi\in K\diamond\psi$. Fix an arbitrary world \mbox{$w\in[K]$}. 

Assume, towards a contradiction, that \mbox{$\min\Big([f_w(\varphi)]\cap C_w,\preceq_w\!\Big)=\varnothing$}, while \mbox{$\min\Big([f_w(\psi)]\cap C_w,\preceq_w\!\Big)\neq\varnothing$}. By condition (SCL), \mbox{$\min\Big([f_w(\psi)]\cap C_w,\preceq_w\!\Big) \subseteq [K\diamond\psi]$}. Since $\varphi\in K\diamond\psi$, we have that $[K\diamond\psi]\subseteq[\varphi]$. Therefore, $\min\Big([f_w(\psi)]\cap C_w,\preceq_w\!\Big) \subseteq [\varphi]$. Moreover, \mbox{$\min\Big([f_w(\psi)]\cap C_w,\preceq_w\!\Big) \subseteq C_w$}. Since $\min\Big([f_w(\psi)]\cap C_w,\preceq_w\!\Big) \neq \varnothing$ by our standing assumption, there exists a world $r\in \min\Big([f_w(\psi)]\cap C_w,\preceq_w\!\Big)$. Consequently, $r\in[\varphi]\cap C_w$, and hence, $[\varphi]\cap C_w\neq\varnothing$. Then, property (F4) entails that $f_w(\varphi)\equiv\varphi$, meaning that $\min\Big([f_w(\varphi)]\cap C_w,\preceq_w\!\Big) = \min\Big([\varphi]\cap C_w,\preceq_w\!\Big)\neq\varnothing$, which is a contradiction. 

By a totally symmetric line of reasoning, we can show that $[\psi]\cap C_w\neq\varnothing$, and that it cannot be the case that \mbox{$\min\Big([f_w(\psi)]\cap C_w,\preceq_w\!\Big)=\varnothing$}, while \mbox{$\min\Big([f_w(\varphi)]\cap C_w,\preceq_w\!\Big)\neq\varnothing$}.

Hence, either $\min\Big([f_w(\varphi)]\cap C_w,\preceq_w\!\Big) = \min\Big([f_w(\psi)]\cap C_w,\preceq_w\!\Big) = \varnothing$, or \mbox{$\min\Big([f_w(\varphi)]\cap C_w,\preceq_w\!\Big)\neq\varnothing$} and \mbox{$\min\Big([f_w(\psi)]\cap C_w,\preceq_w\!\Big)\neq\varnothing$}.

In the former case, obviously $\min\Big([f_w(\varphi)]\cap C_w,\preceq_w\!\Big) = \min\Big([f_w(\psi)]\cap C_w,\preceq_w\!\Big)$.

Now consider the latter case, in which \mbox{$\min\Big([f_w(\varphi)]\cap C_w,\preceq_w\!\Big)\neq\varnothing$} and \mbox{$\min\Big([f_w(\psi)]\cap C_w,\preceq_w\!\Big)\neq\varnothing$}. Since $\psi\in K\diamond\varphi$, condition {\rm(SCL)} gives $\min\Big([f_w(\varphi)]\cap C_w,\preceq_w\Big) \subseteq [K\diamond\varphi] \subseteq[\psi]$. Moreover, the left-hand side of the previous inclusion is a subset of $C_w$ and is non-empty. Consequently, $[\psi]\cap C_w\neq\varnothing$. Similarly, from $\varphi\in K\diamond\psi$ and the non-emptiness of $\min\Big([f_w(\psi)]\cap C_w,\preceq_w\Big)$, we obtain that $[\varphi]\cap C_w\neq\varnothing$. Property (F4) therefore yields $f_w(\psi)\equiv\psi$ and $f_w(\varphi)\equiv\varphi$. Combining the above, we deduce that \mbox{$\min\Big([\varphi]\cap C_w,\preceq_w\!\Big)\subseteq[\psi]\cap C_w$} and \mbox{$\min\Big([\psi]\cap C_w,\preceq_w\!\Big)\subseteq[\varphi]\cap C_w$}. Then, Lemma~\ref{lem_mutual_minimality} of Section~\ref{sec_formal_prel} yields \mbox{$\min\Big([\varphi]\cap C_w,\preceq_w\!\Big)=\min\Big([\psi]\cap C_w,\preceq_w\!\Big)$}, equivalently, \mbox{$\min\Big([f_w(\varphi)]\cap C_w,\preceq_w\!\Big)=\min\Big([f_w(\psi)]\cap C_w,\preceq_w\!\Big)$}. 

Thus, in either case, $\min\Big([f_w(\varphi)]\cap C_w,\preceq_w\!\Big) = \min\Big([f_w(\psi)]\cap C_w,\preceq_w\!\Big)$. Since $w\in[K]$ was arbitrary, this equality holds for every $w\in[K]$. Taking unions and applying condition (SCL), we obtain that $[K\diamond\varphi] = \displaystyle\bigcup_{w\in[K]}\min\Big([f_w(\varphi)]\cap C_w,\preceq_w\!\Big) =$ \mbox{$\displaystyle\bigcup_{w\in[K]}\min\Big([f_w(\psi)]\cap C_w,\preceq_w\!\Big) = [K\diamond\psi]$}. Therefore, $K\diamond\varphi=K\diamond\psi$, as desired.

\item For postulate $(K\diamond8)^{S}$, let $K=Cn(w)$ be complete, and suppose that $\varphi\in K\diamond\varphi$ and $\psi\in K\diamond\psi$. If either $[K\diamond\varphi]=\varnothing$ or $[K\diamond\psi]=\varnothing$, then $Cn\bigl((K\diamond\varphi)\cup(K\diamond\psi)\bigr) = \mathcal{L}$. It follows then trivially that $K\diamond(\varphi\vee\psi) \subseteq Cn\big((K\diamond\varphi)\cup(K\diamond\psi)\big)$.

Suppose, therefore, that $[K\diamond\varphi]\neq\varnothing$ and $[K\diamond\psi]\neq\varnothing$. Since $K=Cn(w)$ is complete, we have that $[K]=\{w\}$. Then, by condition (SCL), $[K\diamond\varphi] = \min\Big([f_w(\varphi)]\cap C_w,\preceq_w\!\Big)$ and $[K\diamond\psi] = \min\Big([f_w(\psi)]\cap C_w,\preceq_w\!\Big)$. Since $\varphi\in K\diamond\varphi$, every world in the non-empty set $[K\diamond\varphi]$ satisfies $\varphi$. Moreover, condition (SCL) ensures that every such world belongs to $C_w$. Consequently, there exists a world in $[\varphi]\cap C_w$, and hence, $[\varphi]\cap C_w\neq\varnothing$. Similarly, from $\psi\in K\diamond\psi$ and the non-emptiness of $[K\diamond\psi]$, we obtain that $[\psi]\cap C_w\neq\varnothing$. It follows immediately that $[\varphi\vee\psi]\cap C_w\neq\varnothing$. Property (F4) therefore applies to all three epistemic inputs, and yields $f_w(\varphi)\equiv\varphi$, $f_w(\psi)\equiv\psi$, and $f_w(\varphi\vee\psi)\equiv\varphi\vee\psi$. Hence, condition (SCL) gives $[K\diamond\varphi] = \min\Big([\varphi]\cap C_w,\preceq_w\Big)$, $[K\diamond\psi] = \min\Big([\psi]\cap C_w,\preceq_w\Big)$, and $[K\diamond(\varphi\vee\psi)] = \min\Big([\varphi\vee\psi]\cap C_w,\preceq_w\Big)$.

Now, it is easy to verify that 
\[
\min\Big([\varphi]\cap C_w,\preceq_w\!\Big)\cap\min\Big([\psi]\cap C_w,\preceq_w\!\Big)\subseteq\min\Big([\varphi\vee\psi]\cap C_w,\preceq_w\!\Big).
\]
Combining the above, we deduce that $[K\diamond\varphi]\cap[K\diamond\psi]\subseteq[K\diamond(\varphi\vee\psi)]$, which is equivalent to $K\diamond(\varphi\vee\psi)\subseteq Cn\bigl((K\diamond\varphi)\cup(K\diamond\psi)\bigr)$, as desired.

\item For postulate $(K\diamond9)$, first note that $[K]\neq\varnothing$ by our standing assumption on the consistency of prior belief sets. For every $w\in[K]$, the theory $Cn(w)$ is complete and $[Cn(w)]=\{w\}$. Applying condition (SCL) to $Cn(w)$ gives
\[
[Cn(w)\diamond\varphi] = \bigcup_{r\in[Cn(w)]}\min\Big([f_r(\varphi)]\cap C_r,\preceq_r\!\Big) = \min\Big([f_w(\varphi)]\cap C_w,\preceq_w\!\Big).
\]
Consequently, $[K\diamond\varphi] = \displaystyle\bigcup_{w\in[K]} [Cn(w)\diamond\varphi]$, which is equivalent to $K\diamond\varphi = \displaystyle\bigcap_{w\in[K]}\bigl(Cn(w)\diamond\varphi\bigr)$. Therefore, postulate $(K\diamond9)$ is satisfied.

\item Finally, for postulate (LC), let $K=Cn(w)$ be complete, suppose that $\varphi\in K\diamond\varphi$, and assume that $\neg\psi\notin K\diamond\varphi$. Since an inconsistent theory contains every sentence of $\mathcal{L}$, $K\diamond\varphi$ is consistent and, therefore, $[K\diamond\varphi]\neq\varnothing$. Since $\neg\psi\notin K\diamond\varphi$, there exists a world $r\in[K\diamond\varphi]$ such that $r\models\psi$. Recall that $K=Cn(w)$ is complete, and hence, $[K]=\{w\}$. By condition (SCL), $[K\diamond\varphi] = \min\Big([f_w(\varphi)]\cap C_w,\preceq_w\Big)$. Therefore, every world in $[K\diamond\varphi]$, and in particular the world $r$, belongs to $C_w$. Since $r\models\psi$, we have that $r\in[\psi]\cap C_w$, and consequently, $[\psi]\cap C_w\neq\varnothing$. Property (F4) now applies and yields $f_w(\psi)\equiv\psi$. Applying condition (SCL) once again, this time to the epistemic input $\psi$, we obtain that $[K\diamond\psi] = \min\Big([f_w(\psi)]\cap C_w,\preceq_w\Big) = \min\Big([\psi]\cap C_w,\preceq_w\Big) \subseteq [\psi]$. Consequently, $\psi\in K\diamond\psi$, as desired.
\end{itemize}

\bigskip

\noindent\underline{\bf Left-to-right implication.} Let $\diamond$ be an update operator that satisfies postulates $(K\diamond1)$, $(K\diamond2)^{S}$, $(K\diamond3)$, $(K\diamond5)$, $(K\diamond6)^{S}$, $(K\diamond7)$, $(K\diamond8)^{S}$, $(K\diamond9)$, and (LC). We show that there exist a credible faithful assignment \mbox{$w\mapsto(C_w,\preceq_w)$} and a pointwise transformation assignment $\mathcal F$ satisfying properties (F1)--(F4), such that condition (SCL) holds.

\paragraph{Construct an Auxiliary Operation.}
First, we define an auxiliary operation $\diamond^{c}$. For every complete theory $Cn(w)$ and every sentence $\varphi\in\mathcal L$, put
\[
Cn(w)\diamond^{c}\varphi =
\begin{cases}
Cn(w)\diamond\varphi,
&
\text{if }\varphi\in Cn(w)\diamond\varphi,
\\[2mm]
\mathcal{L},
&
\text{otherwise}.
\end{cases}
\tag{5}
\]
Moreover, for every consistent belief set $K$, define
\[
K\diamond^{c}\varphi = \bigcap_{w\in[K]} \bigl(Cn(w)\diamond^{c}\varphi\bigr). \tag{6}
\]
For the inconsistent belief set $\mathcal L$, put $\mathcal L\diamond^{c}\varphi=\mathcal L$.

We show that $\diamond^{c}$ is a CL update operator, by showing that it satisfies postulates \mbox{$(K\diamond1)$--$(K\diamond3)$} and $(K\diamond5)$--$(K\diamond9)$; see Definition~\ref{def_cl_operator} of Subsection~\ref{subsec_cl_update}. If $K$ is inconsistent, these postulates are satisfied trivially. Hence, in what follows, it suffices to consider consistent $K$.

\begin{itemize}
\item Postulate $(K\diamond1)$ follows immediately from the definition of $\diamond^{c}$, since both conditions (5) and (6) determine a theory.

\item For postulate \((K\diamond2)\), condition (5) ensures that \(\varphi\in Cn(w)\diamond^{c}\varphi\), for every \(w\in[K]\). Hence, by condition (6), \(\varphi\in K\diamond^{c}\varphi\), as desired.

\item For postulate $(K\diamond3)$, let $\varphi$ be a sentence of $\mathcal{L}$ such that $\varphi\in K$. Fix an arbitrary world $w\in[K]$. Then, $\varphi\in Cn(w)$. By postulate $(K\diamond3)$ for $\diamond$, we have that $Cn(w)\diamond\varphi=Cn(w)$. Hence, \mbox{$Cn(w)\diamond^{c}\varphi=Cn(w)$}, and, using condition (6), we obtain that $K\diamond^{c}\varphi=K$, as desired.

\item For postulate $(K\diamond5)$, let $\varphi$, $\psi$ be two sentences of $\mathcal{L}$ such that $\varphi\equiv\psi$. Fix an arbitrary world $w\in\mathbb M$. By postulate $(K\diamond5)$ for $\diamond$, we have that $Cn(w)\diamond\varphi=Cn(w)\diamond\psi$. Since $\varphi\equiv\psi$, it follows that $\varphi\in Cn(w)\diamond\varphi$ iff $\psi\in Cn(w)\diamond\psi$. Thus, the same clause of condition (5) applies to both epistemic inputs $\varphi$ and $\psi$, meaning that $Cn(w)\diamond^{c}\varphi = Cn(w)\diamond^{c}\psi$. Since $w\in\mathbb M$ was arbitrary, condition (6) entails $K\diamond^{c}\varphi=K\diamond^{c}\psi$, as desired.

\item For postulate $(K\diamond6)$, it is enough to establish the postulate for every complete theory $Cn(w)$. Indeed, by condition~(6), for every belief set $K$ and every sentence $\chi\in\mathcal L$, $[K\diamond^{c}\chi] = \displaystyle\bigcup_{w\in[K]} [Cn(w)\diamond^{c}\chi]$. Suppose, therefore, that, for every $w\in[K]$, $Cn(w)\diamond^{c}(\varphi\wedge\psi) \subseteq \bigl(Cn(w)\diamond^{c}\varphi\bigr)+\psi$. Equivalently, $[Cn(w)\diamond^{c}\varphi]\cap[\psi] \subseteq \big[Cn(w)\diamond^{c}(\varphi\wedge\psi)\big]$. Taking the union over all $w\in[K]$, we obtain that \[ \begin{aligned} \big[(K\diamond^{c}\varphi)+\psi\big] &= [K\diamond^{c}\varphi]\cap[\psi] \\ &= \left( \bigcup_{w\in[K]}[Cn(w)\diamond^{c}\varphi] \right)\cap[\psi] \\ &= \bigcup_{w\in[K]} \Big( [Cn(w)\diamond^{c}\varphi]\cap[\psi] \Big) \\ &\subseteq \bigcup_{w\in[K]} \big[Cn(w)\diamond^{c}(\varphi\wedge\psi)\big] \\ &= \big[K\diamond^{c}(\varphi\wedge\psi)\big]. \end{aligned} \] Hence, $K\diamond^{c}(\varphi\wedge\psi) \subseteq (K\diamond^{c}\varphi)+\psi$. Thus, it remains only to verify the local inclusion $Cn(w)\diamond^{c}(\varphi\wedge\psi) \subseteq \bigl(Cn(w)\diamond^{c}\varphi\bigr)+\psi$, for an arbitrary complete theory $K = Cn(w)$.

To that end, assume first that $K\diamond^{c}\varphi=\mathcal{L}$. Then, the required local inclusion follows trivially. Suppose, therefore, that $K\diamond^{c}\varphi$ is consistent. By condition (5), it follows that $\varphi\in K\diamond\varphi$ and $K\diamond^{c}\varphi=K\diamond\varphi$. Next, we distinguish two cases, according to whether $\varphi\wedge\psi$ belongs to $K\diamond(\varphi\wedge\psi)$.

\begin{itemize}
\item Suppose that $\varphi\wedge\psi \in K\diamond(\varphi\wedge\psi)$. Then, from the first clause of condition~(5), we have that $K\diamond^{c}(\varphi\wedge\psi)=K\diamond(\varphi\wedge\psi)$. Thus, $\diamond^{c}$ coincides with $\diamond$ on both epistemic inputs $\varphi$ and $\varphi\wedge\psi$. Hence, the required local inclusion follows directly from postulate $(K\diamond6)^{S}$ for $\diamond$.

\item Suppose that $\varphi\wedge\psi\notin K\diamond(\varphi\wedge\psi)$.  Then, from condition~(5), we have that \mbox{$K\diamond^{c}(\varphi\wedge\psi)=\mathcal{L}$}.

Now, we claim that $\neg\psi\in K\diamond\varphi$. Assume, towards a contradiction, that $\neg\psi\notin K\diamond\varphi$. Since $\varphi\in K\diamond\varphi$, it follows that $\neg(\varphi\wedge\psi)\notin K\diamond\varphi$. Indeed, if $\neg(\varphi\wedge\psi)\in K\diamond\varphi$, then, by the deductive closure of $K\diamond\varphi$ guaranteed by postulate $(K\diamond1)$ and the fact that $\varphi\wedge\neg(\varphi\wedge\psi)\models\neg\psi$, we would obtain $\neg\psi\in K\diamond\varphi$, contrary to our assumption. Then, applying postulate (LC) with $\varphi\wedge\psi$ in place of $\psi$, we obtain that $\varphi\wedge\psi\in K\diamond(\varphi\wedge\psi)$, contradicting the assumption of the present case. Therefore, $\neg\psi\in K\diamond\varphi$. Consequently, $(K\diamond^{c}\varphi)+\psi=\mathcal{L}$. This, combined with $K\diamond^{c}(\varphi\wedge\psi)=\mathcal{L}$, leads to $K\diamond^{c}(\varphi\wedge\psi) \subseteq (K\diamond^{c}\varphi)+\psi$, as desired.
\end{itemize}

\item For postulate $(K\diamond7)$, it is sufficient to establish the postulate for every complete theory $Cn(w)$, since $\diamond^{c}$ is extended to arbitrary consistent theories pointwise by condition (6). To that end, let $K=Cn(w)$ be complete, and let $\varphi$, $\psi$ be two sentences of $\mathcal{L}$, such that \mbox{$\psi\in K\diamond^{c}\varphi$} and $\varphi\in K\diamond^{c}\psi$. If both $K\diamond^{c}\varphi$ and $K\diamond^{c}\psi$ are inconsistent, then $K\diamond^{c}\varphi=K\diamond^{c}\psi=\mathcal{L}$. Thus, postulate $(K\diamond7)$ trivially holds.

Suppose now that exactly one is inconsistent; say $K\diamond^{c}\varphi=\mathcal{L}$ and $K\diamond^{c}\psi$ is consistent. By condition (5), it follows that $\psi\in K\diamond\psi$ and $K\diamond^{c}\psi=K\diamond\psi$. Since $\varphi\in K\diamond^{c}\psi = K\diamond\psi$ and $K\diamond\psi$ is consistent, we derive that $\neg\varphi\notin K\diamond\psi$. By postulate (LC), $\varphi\in K\diamond\varphi$, and hence, by the first clause of condition (5), $K\diamond^{c}\varphi=K\diamond\varphi=\mathcal{L}$. Therefore, $\psi\in K\diamond\varphi$, and postulate $(K\diamond7)$ applied to $\diamond$ gives $K\diamond\varphi=K\diamond\psi$, contradicting the consistency of $K\diamond\psi$. The opposite mixed case ---i.e., $K\diamond^{c}\psi=\mathcal{L}$ and $K\diamond^{c}\varphi$ is consistent--- can be proved symmetrically, by interchanging $\varphi$ and $\psi$.

Thus, either both $K\diamond^{c}\varphi$ and $K\diamond^{c}\psi$ are inconsistent or both are consistent. In the former case, it was previously shown that postulate $(K\diamond7)$ trivially holds. In the latter case, by condition (5), it follows that $\varphi\in K\diamond\varphi$, $\psi\in K\diamond\psi$, $K\diamond^{c}\varphi=K\diamond\varphi$, and $K\diamond^{c}\psi=K\diamond\psi$. Hence, $\diamond^{c}$ coincides with $\diamond$ on both epistemic inputs $\varphi$ and $\psi$. Moreover, from the standing assumptions $\psi\in K\diamond^{c}\varphi$ and $\varphi\in K\diamond^{c}\psi$, we obtain that $\psi\in K\diamond\varphi$ and $\varphi\in K\diamond\psi$. Therefore, postulate $(K\diamond7)$ applied to $\diamond$ yields $K\diamond\varphi=K\diamond\psi$. Consequently, $K\diamond^{c}\varphi=K\diamond\varphi=K\diamond\psi=K\diamond^{c}\psi$, meaning that postulate $(K\diamond7)$ is satisfied in this case as well.

\item For postulate $(K\diamond8)$, let $K$ be complete, and let $\varphi$, $\psi$ be two sentences of $\mathcal{L}$. If either $K\diamond^{c}\varphi$ or $K\diamond^{c}\psi$ is inconsistent, then $Cn\bigl((K\diamond^{c}\varphi)\cup(K\diamond^{c}\psi)\bigr)=\mathcal{L}$, and it follows trivially that $K\diamond^{c}(\varphi\vee\psi) \subseteq Cn\bigl((K\diamond^{c}\varphi)\cup(K\diamond^{c}\psi)\bigr)$, as required by postulate $(K\diamond8)$. 

Suppose now that both $K\diamond^{c}\varphi$ and $K\diamond^{c}\psi$ are consistent. By condition (5), we have that \mbox{$\varphi\in K\diamond\varphi$}, $\psi\in K\diamond\psi$, $K\diamond^{c}\varphi=K\diamond\varphi$, and $K\diamond^{c}\psi=K\diamond\psi$. Since $K\diamond\varphi$ is consistent and entails $\varphi$, it follows that $\neg(\varphi\vee\psi)\notin K\diamond\varphi$. Then, by postulate (LC), $\varphi\vee\psi\in K\diamond(\varphi\vee\psi)$, and hence, by the first clause of condition (5), $K\diamond^{c}(\varphi\vee\psi)=K\diamond(\varphi\vee\psi)$. Thus, $\diamond^{c}$ coincides with $\diamond$ on $\varphi$, $\psi$, and $\varphi\vee\psi$. Therefore, postulate $(K\diamond8)^{S}$ applied to $\diamond$ yields $K\diamond^{c}(\varphi\vee\psi)\subseteq Cn\bigl((K\diamond^{c}\varphi)\cup(K\diamond^{c}\psi)\bigr)$, as required by postulate $(K\diamond8)$.

\item Finally, postulate $(K\diamond9)$ follows directly from condition~(6).
\end{itemize}

Consequently, we have shown that $\diamond^{c}$ is a CL update operator. Hence, by Theorem~\ref{thm_cl_charact}, it follows that there exists a credible faithful assignment $w\mapsto(C_w,\preceq_w)$, such that
\[
\big[Cn(w)\diamond^{c}\varphi\big] = \min\Big([\varphi]\cap C_w,\preceq_w\!\Big). \tag{7}
\]

\paragraph{Construct a Pointwise Transformation Assignment.}
Thereafter, we construct a pointwise transformation assignment $\mathcal F=\{f_w:w\in\mathbb M\}$. For every world $w\in\mathbb{M}$, define $f_w$ as follows. If \mbox{$\varphi\in Cn(w)\diamond\varphi$}, put
\[
f_w(\varphi)=\varphi. \tag{8}
\]
If \mbox{$\varphi\notin Cn(w)\diamond\varphi$}, use postulate $(K\diamond2)^{S}$ to choose a sentence $\psi\in\mathcal{L}$ such that
\[
\varphi\models\psi,\qquad  \psi\in Cn(w)\diamond\varphi, \qquad Cn(w)\diamond\varphi=Cn(w)\diamond\psi, \tag{9}
\]
and put
\[
f_w(\varphi)=\psi. \tag{10}
\]
The choice in (10) is made uniformly on logical-equivalence classes. This is well defined since postulate $(K\diamond5)$ ensures that logically equivalent epistemic inputs yield the same update result, while the conditions in (9) are invariant under logical equivalence. Hence, the same choices are available for logically equivalent epistemic inputs.

\paragraph{Verify (F1)--(F4).}
We now prove that the pointwise transformation assignment \mbox{$\mathcal F=\{f_w:w\in\mathbb M\}$} satisfies properties (F1)--(F4).

\begin{itemize}
\item Property (F1) follows directly from the definition of $f_w$.

\item For property (F2), let \(\varphi\equiv\chi\). Fix $w\in\mathbb{M}$. By postulate \((K\diamond5)\), \(Cn(w)\diamond\varphi=Cn(w)\diamond\chi\), and therefore, \(\varphi\in Cn(w)\diamond\varphi\) iff \(\chi\in Cn(w)\diamond\chi\). Suppose first that $\varphi\in Cn(w)\diamond\varphi$. Then, $\chi\in Cn(w)\diamond\chi$, and condition (8) yields $f_w(\varphi)=\varphi$ and $f_w(\chi)=\chi$. Hence, $f_w(\varphi)\equiv f_w(\chi)$. Suppose instead that $\varphi\notin Cn(w)\diamond\varphi$. Then, $\chi\notin Cn(w)\diamond\chi$, and the uniform choice in condition (10) ensures that the same sentence is selected for both epistemic inputs. Therefore, $f_w(\varphi)\equiv f_w(\chi)$, and property (F2) follows.

\item For property (F3), if $\varphi\in Cn(w)\diamond\varphi$, then from condition (8) we have that $f_w(\varphi)=\varphi$, and thus, $f_w\big(f_w(\varphi)\big) = f_w(\varphi)$. If, on the other hand, $\varphi\notin Cn(w)\diamond\varphi$, let $f_w(\varphi)=\psi$, as specified by conditions~(9) and~(10). Then, $\psi\in Cn(w)\diamond\psi$, because $Cn(w)\diamond\psi = Cn(w)\diamond\varphi$. Therefore, condition (8) applies to $\psi$, and $f_w\big(f_w(\varphi)\big) = f_w(\psi) = \psi = f_w(\varphi)$.

\item For property (F4), suppose that $[\varphi]\cap C_w\neq\varnothing$. By condition (7), we derive that \mbox{$[Cn(w)\diamond^{c}\varphi]\neq\varnothing$}. Hence, \mbox{$Cn(w)\diamond^{c}\varphi\neq \mathcal{L}$}. Then, by condition (5), we have that \mbox{$\varphi\in Cn(w)\diamond\varphi$}. Therefore, condition (8) yields $f_w(\varphi)=\varphi$.
\end{itemize}

\paragraph{Prove Condition (SCL).}
As a last step, we show that condition (SCL) holds. Fix an arbitrary world $w\in\mathbb{M}$ and a sentence $\varphi\in\mathcal{L}$. We distinguish two cases, according to whether $\varphi$ belongs to $Cn(w)\diamond\varphi$.

\begin{itemize}
\item Suppose that $\varphi\in Cn(w)\diamond\varphi$. Then, by conditions (5) and (8), we derive that \mbox{$Cn(w)\diamond^{c}\varphi = Cn(w)\diamond\varphi$} and $f_w(\varphi)=\varphi$, respectively. Hence, by condition~(7), we deduce that
\[
[Cn(w)\diamond\varphi] = \min\Big([f_w(\varphi)]\cap C_w,\preceq_w\!\Big).
\]

\item Suppose that $\varphi\notin Cn(w)\diamond\varphi$. Let $f_w(\varphi)=\psi$, as specified by conditions~(9) and~(10). Then, \mbox{$Cn(w)\diamond\varphi = Cn(w)\diamond\psi$}. Moreover, $\psi\in Cn(w)\diamond\psi$, and hence, by condition~(5), $Cn(w)\diamond^{c}\psi = Cn(w)\diamond\psi$. Using condition~(7), we obtain that
\[
\begin{aligned}
[Cn(w)\diamond\varphi]
&=
[Cn(w)\diamond\psi] \\[1mm]
&=
[Cn(w)\diamond^{c}\psi] \\[1mm]
&=
\min\Big([\psi]\cap C_w,\preceq_w\!\Big) \\[1mm]
&=
\min\Big([f_w(\varphi)]\cap C_w,\preceq_w\!\Big).
\end{aligned}
\]
\end{itemize}

Finally, in view of our standing assumption on consistent prior belief sets, let $K$ be a consistent belief set. By postulate $(K\diamond9)$ and the equality \mbox{$[Cn(w)\diamond\varphi] = \min\Big([f_w(\varphi)]\cap C_w,\preceq_w\!\Big)$}, established above for every $w\in[K]$, we obtain that $[K\diamond\varphi] = \displaystyle\bigcup_{w\in[K]} [Cn(w)\diamond\varphi] =$ \mbox{$\displaystyle\bigcup_{w\in[K]} \min\Big([f_w(\varphi)]\cap C_w,\preceq_w\!\Big)$}. Thus, $\diamond$ is represented by condition (SCL), as required.
\end{proof}

\subsection{An Unrestricted-Credibility Variant: Selective Belief Update}
\label{subsec_genuine_selective}

The SCL framework combines two conceptually distinct mechanisms; i.e., a source-dependent transformation that determines which part of the epistemic input is accepted, and a credibility restriction that limits the successor worlds available from each source world. To isolate the former mechanism, we now {\em remove the credibility restrictions} by taking every credible set to be the entire set of worlds $\mathbb{M}$. In this unrestricted setting, properties (F1)--(F3) retain the essential structural requirements on selective transformations, whereas property (F4) is deliberately omitted, since it would force every consistent epistemic input to be accepted unchanged.\footnote{Indeed, if \(C_w=\mathbb M\) and \(\varphi\) is consistent, then \([\varphi]\cap C_w=[\varphi]\neq\varnothing\); hence, property \({\rm(F4)}\) entails \(f_w(\varphi)\equiv\varphi\).} Therefore, the resulting framework captures selective belief update independently of credibility limitation. 

\begin{remark}
Let \(\diamond\) be an update operator specified by the semantic clause underlying condition (SCL), by means of a credible faithful assignment \(w\mapsto(C_w,\preceq_w)\) and a pointwise transformation assignment \(\mathcal F=\{f_w:w\in\mathbb M\}\) satisfying properties (F1)--(F3). If \(C_w=\mathbb M\), for every world \(w\in\mathbb M\), then this semantic clause reduces to
\[
[K\diamond\varphi]=\bigcup_{w\in[K]}\min\bigl([f_w(\varphi)],\preceq_w\bigr).
\]
Thus, the epistemic input \(\varphi\) is selectively transformed relative to each source world \(w\), and the resulting proxy \(f_w(\varphi)\) is processed by the standard pointwise update mechanism. In this case, \(\diamond\) implements (genuinely) selective belief update.
\end{remark}

Selective belief update may be regarded as the belief-update analogue of selective belief revision~\cite{ferme99,garapa21} --- whereas selective belief revision transforms the epistemic input through a single function \(f\) before applying revision, the present construction transforms the epistemic input relative to each source world \(w\), and then applies the KM pointwise update mechanism.

\section{Well-Behaved Classes of SCL Update Operators}
\label{sec_classes_scl_operators}

In this section, we identify and characterize two sub-classes of SCL update operators obtained by progressively strengthening the requirements imposed on the pointwise transformation assignment $\mathcal{F}$. {\em Consistency-preserving} SCL update operators, introduced in Subsection~\ref{subsec_sc_update}, require the transformed epistemic input to be credible from its corresponding source world whenever the original epistemic input is consistent. This guarantees that each source world contributes at least one successor to the updated belief set whenever the epistemic input is consistent. {\em Maximal consistency-preserving} SCL update operators, introduced in Subsection~\ref{subsec_max_sc_update}, {\em additionally} require, for every consistent epistemic input, the selected proxy to be {\em maximally informative} among its credible consequences. For each sub-class, we provide both a semantic definition and an axiomatic characterization.

\subsection{Consistency-Preserving SCL Update Operators}
\label{subsec_sc_update}

The general SCL framework does {\em not} require the transformed epistemic input \(f_w(\varphi)\) to be credible from its source world \(w\). Consequently, it may be the case that
\[
[f_w(\varphi)]\cap C_w=\varnothing,
\]
in which case the branch associated with \(w\) contributes no successor to the updated belief set. To exclude this possibility for consistent epistemic inputs, we impose the following additional requirement on the pointwise transformation assignment \(\mathcal F=\{f_w:w\in\mathbb M\}\).

\vspace{3mm}

\begin{tabular}{l p{13cm}}
{\bf (F5)} & If $[\varphi]\neq\varnothing$, then \([f_w(\varphi)]\cap C_w\neq\varnothing\).
\end{tabular}

\vspace{3mm}

Property (F5) requires the transformed epistemic input to be credible from its corresponding source world whenever the epistemic input is consistent. Since the set \(\mathbb M\) of possible worlds is finite, every non-empty subset of \(C_w\) has at least one minimal element under a preorder \(\preceq_w\). Therefore, whenever $[\varphi]\neq\varnothing$, it follows that
\[
\min\Bigl([f_w(\varphi)]\cap C_w,\preceq_w\Bigr)\neq\varnothing.
\] 
Thus, for every consistent epistemic input, every source world contributes at least one selected successor under condition (SCL).

\begin{definition}[Consistency-Preserving SCL Update Operator] \label{def_sc_scl_operator}
An SCL update operator $\diamond$ is a consistency-preserving SCL update operator iff it admits a representation satisfying condition (SCL), in which the pointwise transformation assignment \(\mathcal F=\{f_w:w\in\mathbb M\}\) additionally satisfies property (F5).
\end{definition}

The semantic requirement imposed by property (F5) has a direct axiomatic counterpart. Since every source world contributes at least one successor whenever the epistemic input is consistent, updating a consistent prior belief set by a consistent epistemic input always produces a consistent result. This is expressed by the standard KM consistency postulate \((K\diamond4)\). On that basis, the following theorem strengthens Theorem~\ref{thm_scl_charact} accordingly, and characterizes consistency-preserving SCL update operators.

\begin{theorem} \label{thm_sc_scl_charact}
An update operator $\diamond$ satisfies postulates $(K\diamond1)$, $(K\diamond2)^{S}$, $(K\diamond3)$, $(K\diamond4)$, $(K\diamond5)$, $(K\diamond6)^{S}$, $(K\diamond7)$, $(K\diamond8)^{S}$, $(K\diamond9)$, and (LC) iff there exist a credible faithful assignment \mbox{$w\mapsto(C_w,\preceq_w)$} and a pointwise transformation assignment $\mathcal F=\{f_w:w\in\mathbb M\}$ satisfying properties (F1)--(F5), such that, for every belief set $K$ and every sentence $\varphi\in\mathcal L$, condition (SCL) holds.
\end{theorem}

\begin{proof}
\underline{\bf Right-to-left implication.} Let $\diamond$ be an update operator, let $K$ be a belief set, let $\varphi$ be a sentence of $\mathcal{L}$, and assume that there exist a credible faithful assignment \mbox{$w\mapsto(C_w,\preceq_w)$} and a pointwise transformation assignment $\mathcal F$ satisfying properties (F1)--(F5), such that condition (SCL) holds; that is, \mbox{$[K\diamond\varphi] = \displaystyle\bigcup_{w\in[K]}\min\Big([f_w(\varphi)]\cap C_w\,,\preceq_w\!\Big)$}. We show that $\diamond$ satisfies postulates \mbox{$(K\diamond1)$}, $(K\diamond2)^{S}$, $(K\diamond3)$, $(K\diamond4)$, $(K\diamond5)$, $(K\diamond6)^{S}$, $(K\diamond7)$, $(K\diamond8)^{S}$, $(K\diamond9)$, and (LC).

Since $\mathcal F$ satisfies properties (F1)--(F4), Theorem~\ref{thm_scl_charact} entails that $\diamond$ satisfies postulates $(K\diamond1)$, $(K\diamond2)^{S}$, $(K\diamond3)$, $(K\diamond5)$, $(K\diamond6)^{S}$, $(K\diamond7)$, $(K\diamond8)^{S}$, $(K\diamond9)$, and (LC). It remains to establish postulate $(K\diamond4)$.

To that end, suppose that both $K$ and $\varphi$ are consistent. Then, $[K]\neq\varnothing$. By property (F5), for every $w\in[K]$, $[f_w(\varphi)]\cap C_w\neq\varnothing$. Hence, $\min\Big([f_w(\varphi)]\cap C_w,\preceq_w\!\Big)\neq\varnothing$, for every $w\in[K]$. Then, it follows from condition (SCL) that $[K\diamond\varphi] = \displaystyle\bigcup_{w\in[K]} \min\Big([f_w(\varphi)]\cap C_w,\preceq_w\!\Big) \neq \varnothing$. Therefore, $K\diamond\varphi$ is consistent, and postulate $(K\diamond4)$ follows.

\bigskip

\noindent\underline{\bf Left-to-right implication.} Let $\diamond$ be an update operator that satisfies postulates $(K\diamond1)$, $(K\diamond2)^{S}$, $(K\diamond3)$, $(K\diamond4)$, $(K\diamond5)$, $(K\diamond6)^{S}$, $(K\diamond7)$, $(K\diamond8)^{S}$, $(K\diamond9)$, and (LC). We show that there exist a credible faithful assignment \mbox{$w\mapsto(C_w,\preceq_w)$} and a pointwise transformation assignment $\mathcal F$ satisfying properties (F1)--(F5), such that condition (SCL) holds.

By Theorem~\ref{thm_scl_charact}, there exist a credible faithful assignment $w\mapsto(C_w,\preceq_w)$ and a pointwise transformation assignment $\mathcal F=\{f_w:w\in\mathbb M\}$ satisfying properties (F1)--(F4), such that condition (SCL) holds. It remains to show that the transformation assignment $\mathcal F$ also satisfies property (F5). To that end, fix an arbitrary world $w\in\mathbb M$ and an arbitrary consistent sentence $\varphi\in\mathcal L$. 

Since $Cn(w)$ is consistent, postulate $(K\diamond4)$ entails that $Cn(w)\diamond\varphi$ is consistent, and therefore, $[Cn(w)\diamond\varphi]\neq\varnothing$. Since $[Cn(w)]=\{w\}$, condition (SCL) gives \mbox{$[Cn(w)\diamond\varphi]=\min\Big([f_w(\varphi)]\cap C_w,\preceq_w\!\Big)$}. Consequently, $\min\Big([f_w(\varphi)]\cap C_w,\preceq_w\!\Big)\neq\varnothing$, which implies that $[f_w(\varphi)]\cap C_w\neq\varnothing$. Thus, the transformation assignment $\mathcal F=\{f_w:w\in\mathbb M\}$ satisfies property (F5), as desired.
\end{proof}

The following lemma provides a behavioural characterization of the credible sets occurring in the representations supplied by Theorem~\ref{thm_sc_scl_charact}. Intuitively, it shows that a consistent sentence $\varphi$ is credible from a source world $w$ exactly when updating the complete theory $Cn(w)$ by $\varphi$ results in a belief set that accepts $\varphi$. Thus, the semantic condition $[\varphi]\cap C_w\neq\varnothing$ can be characterized entirely in terms of the observable behaviour of an update operator $\diamond$. This equivalence will be used in Subsection~\ref{subsec_max_sc_update} to characterize maximal consistency-preserving SCL update operators.

\begin{lemma} \label{lem_local_credibility}
Let $w\mapsto(C_w,\preceq_w)$ and $\mathcal F=\{f_w:w\in\mathbb M\}$ be assignments witnessing a representation of $\diamond$ as in Theorem~\ref{thm_sc_scl_charact}. Then, for every world $w\in\mathbb M$ and every consistent sentence $\varphi\in\mathcal L$,
\[
[\varphi]\cap C_w\neq\varnothing \quad\text{iff}\quad \varphi\in Cn(w)\diamond\varphi.
\]
\end{lemma}

\begin{proof}
Let $w$ be a world of $\mathbb{M}$ and let $\varphi$ be a consistent sentence of $\mathcal{L}$. 

First, suppose that $[\varphi]\cap C_w\neq\varnothing$. By property (F4), $f_w(\varphi)\equiv\varphi$. Since $[Cn(w)]=\{w\}$, condition (SCL) gives \mbox{$[Cn(w)\diamond\varphi] = \min\Big([\varphi]\cap C_w,\preceq_w\!\Big)\subseteq[\varphi]$}. Therefore, $\varphi\in Cn(w)\diamond\varphi$.

Conversely, suppose that $\varphi\in Cn(w)\diamond\varphi$. By property (F5), $[f_w(\varphi)]\cap C_w\neq\varnothing$. Hence, condition (SCL) yields \mbox{$[Cn(w)\diamond\varphi] = \min\Big([f_w(\varphi)]\cap C_w,\preceq_w\!\Big)\neq\varnothing$}. Let $r\in[Cn(w)\diamond\varphi]$. Since $\varphi\in Cn(w)\diamond\varphi$, we have $r\models\varphi$, while condition (SCL) ensures that $r\in C_w$. Consequently, $r\in[\varphi]\cap C_w$, and therefore, $[\varphi]\cap C_w\neq\varnothing$.
\end{proof}

\subsection{Maximal Consistency-Preserving SCL Update Operators}
\label{subsec_max_sc_update}

Theorem~\ref{thm_sc_scl_charact} characterizes consistency-preserving SCL update operators by requiring every transformed epistemic input \(f_w(\varphi)\) to be credible from its corresponding source world \(w\), whenever the original epistemic input $\varphi$ is consistent. This requirement guarantees the existence of a credible proxy, but does not determine how informative that proxy must be. In particular, when several credible consequences of \(\varphi\) are available, properties (F1)--(F5) permit the transformation to select a proxy even though a strictly stronger credible consequence could have been retained. To exclude such unnecessarily weak selections, we impose the following maximality requirement.

\vspace{3mm}

\begin{tabular}{l p{13cm}}
{\bf (F6)} & There is no sentence \(\psi\in\mathcal L\) such that \(\varphi\models\psi\), \(\psi\models f_w(\varphi)\), \(f_w(\varphi)\not\models\psi\), and \([\psi]\cap C_w\neq\varnothing\).
\end{tabular}

\vspace{3mm}

Properties (F5) and (F6) jointly require $f_w(\varphi)$ to be maximally informative among the credible consequences of $\varphi$. More precisely, property (F5) ensures that $f_w(\varphi)$ is credible from $w$, while property (F6) excludes the existence of a credible consequence $\psi$ of $\varphi$ that is strictly stronger than $f_w(\varphi)$ and lies logically between $\varphi$ and $f_w(\varphi)$. The requirement concerns maximality rather than the existence of a uniquely strongest credible consequence; thus, several mutually incomparable maximal proxies may be available.

\begin{definition}[Maximal Consistency-Preserving SCL Update Operator] \label{def_max_sc_scl_operator}
A consistency-preserving SCL update operator \(\diamond\) is a maximal consistency-preserving SCL update operator iff it admits a representation satisfying condition (SCL), in which the pointwise transformation assignment \mbox{\(\mathcal F=\{f_w:w\in\mathbb M\}\)} additionally satisfies property (F6).
\end{definition}

We next formulate the axiomatic counterpart of this semantic requirement. By Lemma~\ref{lem_local_credibility} of the previous subsection, for a complete theory \(K=Cn(w)\), the condition $\chi\in K\diamond\chi$ holds exactly when \(\chi\) is credible from the source world \(w\). Consequently, the credibility condition \([\chi]\cap C_w\neq\varnothing\) occurring in property (F6) can be expressed entirely in terms of the observable behaviour of the update operator $\diamond$. This leads to the following strengthening of postulate \((K\diamond2)^S\). In addition to requiring an accepted proxy that is implied by the epistemic input and induces the same update result, postulate $(K\diamond2)^{SM}$ requires that proxy to be maximally informative among the locally credible consequences of the epistemic input.

\vspace{4mm}

\noindent \begin{tabular}{l p{12cm}}
$\mathbf{(K\diamond2)^{SM}}$ & If $K$ is complete, then there exists a sentence $\psi\in\mathcal{L}$ such that $\varphi\models\psi$, \mbox{$\psi\in K\diamond\varphi$}, $K\diamond\varphi=K\diamond\psi$, and, for every $\chi\in\mathcal{L}$, if $\varphi\models\chi$, $\chi\models\psi$, $\chi\in K\diamond\chi$, then $\psi\models\chi$.
\end{tabular}

\vspace{4mm}

As stated, postulate \((K\diamond2)^{SM}\) strengthens \((K\diamond2)^S\). Its first three requirements provide an accepted proxy \(\psi\) that is a consequence of \(\varphi\) and produces the same update result. Its final requirement imposes local maximality. If a locally credible consequence \(\chi\) lies between \(\varphi\) and \(\psi\), then \(\psi\models\chi\); since \(\chi\models\psi\) is already assumed, it follows that \(\chi\equiv\psi\). Thus, no strictly stronger locally credible consequence of \(\varphi\) can lie between \(\varphi\) and the selected proxy.

Against this background, we can formulate the representation theorem for maximal consistency-preserving SCL update operators.

\begin{theorem} \label{thm_max_sc_scl_charact}
An update operator $\diamond$ satisfies postulates $(K\diamond1)$, $(K\diamond2)^{SM}$, $(K\diamond3)$, $(K\diamond4)$, $(K\diamond5)$, $(K\diamond6)^{S}$, $(K\diamond7)$, $(K\diamond8)^{S}$, $(K\diamond9)$, and (LC) iff there exist a credible faithful assignment \mbox{$w\mapsto(C_w,\preceq_w)$} and a pointwise transformation assignment $\mathcal F=\{f_w:w\in\mathbb M\}$ satisfying properties (F1)--(F6), such that, for every belief set $K$ and every sentence $\varphi\in\mathcal L$, condition (SCL) holds.
\end{theorem}

\begin{proof}
\noindent\underline{\bf Right-to-left implication.} Let $\diamond$ be an update operator, let $K$ be a belief set, let $\varphi$ be a sentence of $\mathcal{L}$, and assume that there exist a credible faithful assignment \mbox{$w\mapsto(C_w,\preceq_w)$} and a pointwise transformation assignment $\mathcal F$ satisfying properties (F1)--(F6), such that condition (SCL) holds; that is, \mbox{$[K\diamond\varphi] = \displaystyle\bigcup_{w\in[K]}\min\Big([f_w(\varphi)]\cap C_w\,,\preceq_w\!\Big)$}. We show that $\diamond$ satisfies postulates \mbox{$(K\diamond1)$}, $(K\diamond2)^{SM}$, $(K\diamond3)$, $(K\diamond4)$, $(K\diamond5)$, $(K\diamond6)^{S}$, $(K\diamond7)$, $(K\diamond8)^{S}$, $(K\diamond9)$, and (LC).

Since $\mathcal F$ satisfies properties (F1)--(F5), Theorem~\ref{thm_sc_scl_charact} entails that $\diamond$ satisfies postulates $(K\diamond1)$, $(K\diamond2)^S$, $(K\diamond3)$, $(K\diamond4)$, $(K\diamond5)$, $(K\diamond6)^S$, $(K\diamond7)$, $(K\diamond8)^S$, $(K\diamond9)$, and (LC). It remains to establish postulate $(K\diamond2)^{SM}$.

Let $K=Cn(w)$ be complete, and put $\psi=f_w(\varphi)$. By property (F1),  we have that $\varphi\models\psi$. Since $[K]=\{w\}$, condition (SCL) gives $[K\diamond\varphi] = \min\Bigl([f_w(\varphi)]\cap C_w,\preceq_w\Bigr) \subseteq[\psi]$. Hence, $\psi\in K\diamond\varphi$. Moreover, by property (F3), we have that $f_w(\psi) = f_w\bigl(f_w(\varphi)\bigr) \equiv f_w(\varphi) = \psi$. Therefore,
\[
[K\diamond\psi] = \min\Bigl([f_w(\psi)]\cap C_w,\preceq_w\Bigr) = \min\Bigl([\psi]\cap C_w,\preceq_w\Bigr) =
[K\diamond\varphi].
\]
Thus, $K\diamond\varphi=K\diamond\psi$.

Now, let $\chi\in\mathcal L$ be such that $\varphi\models\chi$, $\chi\models\psi$, and $\chi\in K\diamond\chi$. We show that $\psi\models\chi$. We distinguish two cases.

\begin{itemize}
\item Suppose first that $\chi$ is consistent. By property (F5), we derive that $[f_w(\chi)]\cap C_w\neq\varnothing$. Hence, $\min\Bigl([f_w(\chi)]\cap C_w,\preceq_w\Bigr)\neq\varnothing$. By condition (SCL), it follows that $[K\diamond\chi] = \min\Bigl([f_w(\chi)]\cap C_w,\preceq_w\Bigr)$. Choose a world $r\in[K\diamond\chi]$. Since $\chi\in K\diamond\chi$, we have that $r\models\chi$, while condition (SCL) ensures that $r\in C_w$. Consequently, $[\chi]\cap C_w\neq\varnothing$. Then, property (F6), applied to $\varphi$, $\chi$, and $f_w(\varphi)=\psi$, yields $\psi\models\chi$.

\item Suppose now that $\chi$ is inconsistent. Since $\varphi\models\chi$, the sentence $\varphi$ is also inconsistent, and hence, $\varphi\equiv\chi$. By property (F2), we have that $f_w(\varphi)\equiv f_w(\chi)$, and therefore, $\psi\equiv f_w(\chi)$.

Since $\chi\in K\diamond\chi$ and $\chi$ is inconsistent, the theory $K\diamond\chi$ is inconsistent. Hence $[K\diamond\chi]=\varnothing$. By condition (SCL), it follows that $\min\Bigl([f_w(\chi)]\cap C_w,\preceq_w\Bigr)=\varnothing$. Hence, every non-empty subset of $C_w$ has a minimal element under $\preceq_w$. Consequently, $[f_w(\chi)]\cap C_w=\varnothing$. Since $\psi\equiv f_w(\chi)$, we obtain that $[\psi]\cap C_w=\varnothing$.

Suppose, towards a contradiction, that $\psi$ is consistent. By property (F5), it follows that $[f_w(\psi)]\cap C_w\neq\varnothing$. However, property (F3) gives $f_w(\psi) = f_w\bigl(f_w(\varphi)\bigr) \equiv f_w(\varphi) = \psi$, and therefore, $[\psi]\cap C_w\neq\varnothing$, contrary to the conclusion above. Hence, $\psi$ is inconsistent. It follows  immediately that $\psi\models\chi$.
\end{itemize}

Thus, in either case, $\psi\models\chi$. Therefore, the sentence $\psi$  satisfies all the requirements of postulate $(K\diamond2)^{SM}$, and the postulate follows.

\bigskip

\noindent\underline{\bf Left-to-right implication.} Let $\diamond$ be an update operator that satisfies postulates $(K\diamond1)$, $(K\diamond2)^{SM}$, $(K\diamond3)$, $(K\diamond4)$, $(K\diamond5)$, $(K\diamond6)^{S}$, $(K\diamond7)$, $(K\diamond8)^{S}$, $(K\diamond9)$, and (LC). We show that there exist a credible faithful assignment \mbox{$w\mapsto(C_w,\preceq_w)$} and a pointwise transformation assignment $\mathcal F$ satisfying properties (F1)--(F6), such that condition (SCL) holds.

Since postulate $(K\diamond2)^{SM}$ implies postulate \mbox{$(K\diamond2)^S$}, Theorem~\ref{thm_sc_scl_charact} provides a credible faithful assignment $w\mapsto(C_w,\preceq_w)$ and a pointwise transformation assignment \mbox{$\mathcal G=\{g_w:w\in\mathbb M\}$} satisfying properties (F1)--(F5), such that, for every belief set $K$ and every sentence $\varphi\in\mathcal{L}$, 
\[
[K\diamond\varphi]=\displaystyle\bigcup_{w\in[K]}\min\Big([g_w(\varphi)]\cap C_w,\preceq_w\!\Big). \tag{1}
\] 
Since $w\mapsto(C_w,\preceq_w)$ and $\mathcal G$ witness a representation of $\diamond$ as in Theorem~\ref{thm_sc_scl_charact}, Lemma~\ref{lem_local_credibility} yields, for every world $w\in\mathbb M$ and every consistent sentence $\varphi\in\mathcal L$,
\[
[\varphi]\cap C_w\neq\varnothing \quad\text{iff}\quad \varphi\in Cn(w)\diamond\varphi. \tag{2}
\]

\paragraph{Construct a New Pointwise Transformation Assignment.}
We shall use the above equivalence to construct a {\em new} pointwise transformation assignment satisfying the additional maximality property (F6). To that end, for each $w\in\mathbb M$, define a new transformation function $\widehat f_w$ as follows:
\[
\widehat f_w(\varphi)=
\begin{cases}
\varphi, & \text{if }\ \varphi\in Cn(w)\diamond\varphi,\\[1mm]
\psi, & \text{otherwise},
\end{cases}
\]
where, in the second case, $\psi$ is supplied by postulate $(K\diamond2)^{SM}$ so that
\[
\varphi\models\psi,\qquad \psi\in Cn(w)\diamond\varphi,\qquad Cn(w)\diamond\varphi=Cn(w)\diamond\psi,
\]
and, for every $\chi\in\mathcal L$, if $\varphi\models\chi$, $\chi\models\psi$, and $\chi\in Cn(w)\diamond\chi$, then $\psi\models\chi$. The choices in the second clause are made uniformly on logical-equivalence classes. This is well defined because postulate $(K\diamond5)$ ensures that logically equivalent epistemic inputs yield the same update result, while the entailment and maximality conditions in postulate $(K\diamond2)^{SM}$ are invariant under logical equivalence. Hence, the same witnesses are available for logically equivalent epistemic inputs.

\paragraph{Verify (F1)--(F6).}
Now, we show that the pointwise transformation assignment \mbox{$\widehat{\mathcal F}=\{\widehat f_w:w\in\mathbb M\}$} satisfies properties (F1)--(F6). 

\begin{itemize}
\item Property (F1) follows immediately from the definition of $\widehat f_w$.

\item For property (F2), let $\varphi\equiv\psi$. By postulate $(K\diamond5)$, $Cn(w)\diamond\varphi=Cn(w)\diamond\psi$, and therefore, $\varphi\in Cn(w)\diamond\varphi$ iff $\psi\in Cn(w)\diamond\psi$. Hence, the same clause in the definition of $\widehat f_w$ applies to both epistemic inputs. If the first clause applies, then $\widehat f_w(\varphi)=\varphi\equiv\psi=\widehat f_w(\psi)$. If the second clause applies, the uniform choice on logical-equivalence classes ensures that the same sentence is selected for both epistemic inputs. Thus, $\widehat f_w(\varphi)\equiv\widehat f_w(\psi)$.

\item For property (F3), if $\varphi\in Cn(w)\diamond\varphi$, then from the first clause of the definition of $\widehat f_w$ we have that $\widehat f_w(\varphi)=\varphi$, and thus, $\widehat f_w\big(\widehat f_w(\varphi)\big) = \widehat f_w(\varphi)$. If, on the other hand, $\varphi\notin Cn(w)\diamond\varphi$, let $\widehat f_w(\varphi)=\psi$, as specified in the definition of $\widehat f_w$. Since $\psi\in Cn(w)\diamond\varphi$ and $Cn(w)\diamond\varphi=Cn(w)\diamond\psi$, it follows that $\psi\in Cn(w)\diamond\psi$. Hence, the first clause of the definition of $\widehat f_w$ applies to $\psi$, and gives $\widehat f_w\big(\widehat f_w(\varphi)\big)=\widehat f_w(\psi)=\psi=\widehat f_w(\varphi)$.

\item For property (F4), suppose that $[\varphi]\cap C_w\neq\varnothing$. By condition~(2), we derive that \mbox{$\varphi\in Cn(w)\diamond\varphi$}, and the first clause of the definition of $\widehat f_w$ entails that $\widehat f_w(\varphi)=\varphi$. 

\item For property (F5), let $\varphi$ be consistent. Suppose first that $\varphi\in Cn(w)\diamond\varphi$. Then, the first clause of the definition of $\widehat f_w$ entails $\widehat f_w(\varphi)=\varphi$, while condition~(2) gives $[\widehat f_w(\varphi)]\cap C_w\neq\varnothing$, as required. Suppose instead that $\varphi\notin Cn(w)\diamond\varphi$, and let $\widehat f_w(\varphi)=\psi$, as specified in the definition of $\widehat f_w$. Since $\varphi\models\psi$ and $\varphi$ is consistent, $\psi$ is consistent. Moreover, \mbox{$\psi\in Cn(w)\diamond\varphi = Cn(w)\diamond\psi$}, and hence, $\psi\in Cn(w)\diamond\psi$. Therefore, condition~(2) yields $[\psi]\cap C_w\neq\varnothing$. Consequently, $[\widehat f_w(\varphi)]\cap C_w\neq\varnothing$, and property (F5) follows.

\item For property (F6), suppose first that $\varphi\in Cn(w)\diamond\varphi$. Then, the first clause of the definition of $\widehat f_w$ entails that $\widehat f_w(\varphi)=\varphi$, and there cannot be a sentence $\psi\in\mathcal{L}$ such that both $\varphi\models\psi$ and $\widehat f_w(\varphi)\not\models\psi$. Suppose, therefore, that $\varphi\notin Cn(w)\diamond\varphi$. Assume, towards a contradiction, that there exists a sentence $\psi\in\mathcal L$ such that $\varphi\models\psi$, $\psi\models\widehat f_w(\varphi)$, $\widehat f_w(\varphi)\not\models\psi$, and $[\psi]\cap C_w\neq\varnothing$. Since $[\psi]\cap C_w\neq\varnothing$, the sentence $\psi$ is consistent. Hence, condition~(2) gives $\psi\in Cn(w)\diamond\psi$. Then, postulate $(K\diamond2)^{SM}$ gives $\widehat f_w(\varphi)\models\psi$, contradicting $\widehat f_w(\varphi)\not\models\psi$. Consequently, property (F6) holds.
\end{itemize}
 
\paragraph{Prove Condition (SCL).} As a last step, we show that condition (SCL) holds. Fix an arbitrary world $w\in\mathbb M$ and a sentence $\varphi\in\mathcal L$. Since $[Cn(w)]=\{w\}$, condition~(1) gives, for every sentence $\theta\in\mathcal L$, 
\[ 
[Cn(w)\diamond\theta] = \min\Bigl([g_w(\theta)]\cap C_w,\preceq_w\Bigr). \tag{3} 
\] 
We distinguish two cases, according to whether $\varphi$ belongs to $Cn(w)\diamond\varphi$.

\begin{itemize} \item Suppose first that $\varphi\in Cn(w)\diamond\varphi$. By the first clause of the definition of $\widehat f_w$, we have that $\widehat f_w(\varphi)=\varphi$. Suppose first that $\varphi$ is consistent. By condition~(2), $[\varphi]\cap C_w\neq\varnothing$. Therefore, property (F4) for $g_w$ gives $g_w(\varphi)\equiv\varphi$. Consequently, by condition~(3), 
\[ 
\begin{aligned} [Cn(w)\diamond\varphi] &= \min\Bigl([g_w(\varphi)]\cap C_w,\preceq_w\Bigr) \\ &= \min\Bigl([\varphi]\cap C_w,\preceq_w\Bigr) \\ &= \min\Bigl([\widehat f_w(\varphi)]\cap C_w,\preceq_w\Bigr). \end{aligned} 
\] 
Suppose instead that $\varphi$ is inconsistent. Since $\varphi\in Cn(w)\diamond\varphi$, the theory $Cn(w)\diamond\varphi$ is inconsistent. Hence, $[Cn(w)\diamond\varphi]=\varnothing$. Moreover, $\widehat f_w(\varphi)=\varphi$ and $[\varphi]=\varnothing$. Therefore, $\min\Bigl([\widehat f_w(\varphi)]\cap C_w,\preceq_w\Bigr) = \varnothing = [Cn(w)\diamond\varphi]$. 

\item Suppose now that $\varphi\notin Cn(w)\diamond\varphi$. Let $\widehat f_w(\varphi)=\psi$, as specified by the second clause of the definition of $\widehat f_w$. Then, $Cn(w)\diamond\varphi=Cn(w)\diamond\psi$ and $\psi\in Cn(w)\diamond\varphi = Cn(w)\diamond\psi$. Suppose first that $\psi$ is consistent. By condition~(2), $[\psi]\cap C_w\neq\varnothing$. Therefore, property (F4) for $g_w$ gives $g_w(\psi)\equiv\psi$. Consequently, by condition~(3), 
\[ 
\begin{aligned} [Cn(w)\diamond\varphi] &= [Cn(w)\diamond\psi] \\ &= \min\Bigl([g_w(\psi)]\cap C_w,\preceq_w\Bigr) \\ &= \min\Bigl([\psi]\cap C_w,\preceq_w\Bigr) \\ &= \min\Bigl([\widehat f_w(\varphi)]\cap C_w,\preceq_w\Bigr). \end{aligned} 
\] 
Suppose instead that $\psi$ is inconsistent. Since $\psi\in Cn(w)\diamond\psi$, the theory $Cn(w)\diamond\psi$ is inconsistent. Hence, $[Cn(w)\diamond\psi]=\varnothing$. Since $Cn(w)\diamond\varphi=Cn(w)\diamond\psi$, it follows that $[Cn(w)\diamond\varphi]=\varnothing$. Moreover, $\widehat f_w(\varphi)=\psi$ and $[\psi]=\varnothing$. Therefore, \mbox{$\min\Bigl([\widehat f_w(\varphi)]\cap C_w,\preceq_w\Bigr) = \varnothing = [Cn(w)\diamond\varphi]$}. 
\end{itemize} 

Finally, in view of our standing assumption on consistent prior belief sets, let $K$ be a consistent belief set. By postulate $(K\diamond9)$ and the equality \mbox{$[Cn(w)\diamond\varphi]=\min\Big([\widehat f_w(\varphi)]\cap C_w,\preceq_w\!\Big)$}, established above for every $w\in[K]$, we obtain that $[K\diamond\varphi] = \displaystyle\bigcup_{w\in[K]} [Cn(w)\diamond\varphi] =$ \mbox{$\displaystyle\bigcup_{w\in[K]} \min\Big([\widehat f_w(\varphi)]\cap C_w,\preceq_w\!\Big)$}. Thus, $\diamond$ is represented by condition (SCL), as required.
\end{proof}

\section{KM, CL, and CCL Belief Update as Special Cases of SCL Belief Update}
\label{sec_cl_ccl_special_cases}

Having introduced the SCL framework and characterized its principal sub-classes, we now examine its relationship with standard KM belief update \cite{katsuno92} and the credibility-limited approaches of \mbox{Ferm{\'e} \emph{et al.}~\cite{ferme23}}. We show that CL and CCL belief update arise as special cases of SCL belief update through particular choices of the pointwise transformation assignment, while KM belief update is recovered by additionally removing the credibility restrictions. These results establish that the SCL framework {\em subsumes all three approaches}.

\subsection{CL and CCL Belief Update as Special Cases}
\label{subsec_cl_special_case}

Fix a credible faithful assignment \(w\mapsto(C_w,\preceq_w)\). For every pointwise transformation assignment \(\mathcal F=\{f_w:w\in\mathbb M\}\), let \(\diamond^{\mathcal F}\) denote the operation determined by condition (SCL), relative to this credible faithful assignment. We first show that CL belief update is recovered by taking every source-dependent transformation function to be the identity function.

\begin{proposition} \label{prop_cl_limiting_case}
For every \(w\in\mathbb M\) and every \(\varphi\in\mathcal L\), define the transformation function \mbox{\(f_w^{\mathrm{CL}}(\varphi)=\varphi\)}, and let \mbox{\(\mathcal F_{\mathrm{CL}}=\big\{f_w^{\mathrm{CL}}:w\in\mathbb M\big\}\)} be the corresponding pointwise transformation assignment. Then, \(\mathcal F_{\mathrm{CL}}\) satisfies properties (F1)--(F4), and \[ [K\diamond^{\mathcal F_{\mathrm{CL}}}\varphi]=\bigcup_{w\in[K]}\min\Big([\varphi]\cap C_w,\preceq_w\!\Big). \] Hence, \(\diamond^{\mathcal F_{\mathrm{CL}}}\) is a CL update operator.
\end{proposition}

\begin{proof} Since every \(f_w^{\mathrm{CL}}\) is the identity function, properties (F1)--(F4) follow immediately. By condition (SCL), we have that, for every belief set $K$ and every sentence $\varphi\in\mathcal{L}$, \[ [K\diamond^{\mathcal F_{\mathrm{CL}}}\varphi] =\bigcup_{w\in[K]}\min\Big(\big[f_w^{\mathrm{CL}}(\varphi)\big]\cap C_w,\preceq_w\!\Big) = \bigcup_{w\in[K]}\min\Big([\varphi]\cap C_w,\preceq_w\!\Big). \] This is exactly condition (CL) of Subsection~\ref{subsec_cl_update}. Therefore, by Theorem~\ref{thm_cl_charact}, \(\diamond^{\mathcal F_{\mathrm{CL}}}\) is a CL update operator. 
\end{proof}

Notice that the pointwise transformation assignment \(\mathcal F_{\mathrm{CL}}\) need not satisfy property (F5). Indeed, it may be the case that \([\varphi]\neq\varnothing\), while \([\varphi]\cap C_w=\varnothing\). In that case, $\bigl[f_w^{\mathrm{CL}}(\varphi)\bigr]\cap C_w = [\varphi]\cap C_w = \varnothing$, and the source branch associated with \(w\) contributes no successor world to the updated state of belief.

We next turn to CCL belief update. Its source-retention behaviour can be recovered by leaving locally credible epistemic inputs unchanged and replacing each locally non-credible epistemic input \(\varphi\) with \(\varphi\lor\gamma_w\), where \(\gamma_w\) is a complete sentence characterizing the source world \(w\). Relative to $C_w$, this proxy has $w$ as its unique credible world and is maximally informative among the credible consequences of $\varphi$.

\begin{proposition}\label{prop_ccl_limiting_case}
For every \(w\in\mathbb M\), let \(\gamma_w\) be a sentence such that \([\gamma_w]=\{w\}\), and define the transformation function
\[
f_w^{\mathrm{CCL}}(\varphi)=
\begin{cases}
\varphi, & \text{if }[\varphi]\cap C_w\neq\varnothing,\\[1mm]
\varphi\lor\gamma_w, & \text{otherwise}.
\end{cases}
\]
Let \(\mathcal F_{\mathrm{CCL}}=\{f_w^{\mathrm{CCL}}:w\in\mathbb M\}\) be the corresponding pointwise transformation assignment. Then, \(\mathcal F_{\mathrm{CCL}}\) satisfies properties (F1)--(F6), and
\[
\min\Bigl(\bigl[f_w^{\mathrm{CCL}}(\varphi)\bigr]\cap C_w,\preceq_w\Bigr)=
\begin{cases}
\min\Bigl([\varphi]\cap C_w,\preceq_w\Bigr), & \text{if }\ [\varphi]\cap C_w\neq\varnothing,\\[2mm]
\{w\}, & \text{otherwise}.
\end{cases}
\]
Hence, \(\diamond^{\mathcal F_{\mathrm{CCL}}}\) is a CCL update operator and, in particular, a maximal consistency-preserving SCL update operator.
\end{proposition}

\begin{proof}
First, we show that the pointwise transformation assignment \(\mathcal F_{\mathrm{CCL}}\) satisfies properties \mbox{(F1)--(F6)}.

\begin{itemize}
\item Property (F1) follows because \(f_w^{\mathrm{CCL}}(\varphi)\) is either \(\varphi\) or \(\varphi\lor\gamma_w\), and \(\varphi\models\varphi\lor\gamma_w\). 

\item Property (F2) follows because the condition \([\varphi]\cap C_w\neq\varnothing\) is invariant under logical equivalence.

\item For property (F3), suppose first that \([\varphi]\cap C_w\neq\varnothing\). Then, \(f_w^{\mathrm{CCL}}(\varphi)=\varphi\), and thus, we have immediately that $f_w^{\mathrm{CCL}}\big(f_w^{\mathrm{CCL}}(\varphi)\big) = f_w^{\mathrm{CCL}}(\varphi)$. Suppose instead that \([\varphi]\cap C_w=\varnothing\). Since \(w\in C_w\), we have that $[\varphi\lor\gamma_w]\cap C_w=\{w\}\neq\varnothing$. Hence,
\[
f_w^{\mathrm{CCL}}\bigl(f_w^{\mathrm{CCL}}(\varphi)\bigr) =f_w^{\mathrm{CCL}}(\varphi\lor\gamma_w) =\varphi\lor\gamma_w =f_w^{\mathrm{CCL}}(\varphi).
\]

\item Property (F4) follows directly from the first clause of the definition of the transformation function $f_w^{\mathrm{CCL}}$.

\item Property (F5) follows because
\[
\bigl[f_w^{\mathrm{CCL}}(\varphi)\bigr]\cap C_w=
\begin{cases}
[\varphi]\cap C_w, & \text{if }\ [\varphi]\cap C_w\neq\varnothing,\\
\{w\}, & \text{otherwise},
\end{cases}
\]
and both sets are non-empty in their respective cases. Thus, property (F5) holds.

\item For property (F6), suppose first that \([\varphi]\cap C_w\neq\varnothing\). Then, \(f_w^{\mathrm{CCL}}(\varphi)=\varphi\), so there cannot be a sentence \(\psi\) such that \(\varphi\models\psi\) and \(f_w^{\mathrm{CCL}}(\varphi)\not\models\psi\). Suppose instead that \([\varphi]\cap C_w=\varnothing\), so that \(f_w^{\mathrm{CCL}}(\varphi)=\varphi\lor\gamma_w\). Let \(\psi\) satisfy $\varphi\models\psi$, $\psi\models\varphi\lor\gamma_w$, and $[\psi]\cap C_w\neq\varnothing$. Choose \(r\in[\psi]\cap C_w\). Since \(\psi\models\varphi\lor\gamma_w\) and \([\varphi]\cap C_w=\varnothing\), it follows that \(r=w\). Hence, \(w\models\psi\). Since \([\gamma_w]=\{w\}\), it follows that \(\gamma_w\models\psi\). Together with \(\varphi\models\psi\), this yields \(\varphi\lor\gamma_w\models\psi\). Therefore, no strictly stronger credible consequence of \(\varphi\) lies between \(\varphi\) and \(f_w^{\mathrm{CCL}}(\varphi)\), and property (F6) follows.
\end{itemize}

Finally, if \([\varphi]\cap C_w\neq\varnothing\), then \(f_w^{\mathrm{CCL}}(\varphi)=\varphi\), and hence, 
\[ 
\min\Big(\big[f_w^{\mathrm{CCL}}(\varphi)\big]\cap C_w,\preceq_w\!\Big)=\min\Big([\varphi]\cap C_w,\preceq_w\!\Big). \]
If, on the other hand, \([\varphi]\cap C_w=\varnothing\), then $\bigl[f_w^{\mathrm{CCL}}(\varphi)\bigr]\cap C_w =[\varphi\lor\gamma_w]\cap C_w =\{w\}$, and therefore,
\[
\min\Bigl(\bigl[f_w^{\mathrm{CCL}}(\varphi)\bigr]\cap C_w,\preceq_w\Bigr)=\{w\}.
\]
Consequently, in view of condition (SCL), we obtain that, for every belief set $K$ and every sentence $\varphi\in\mathcal{L}$, \[ \big[K\diamond^{\mathcal F_{\mathrm{CCL}}}\varphi\big]=\bigcup_{w\in[K]} \begin{cases} \min\Big([\varphi]\cap C_w,\preceq_w\!\Big), & \text{if }[\varphi]\cap C_w\neq\varnothing,\\[1mm] \{w\}, & \text{otherwise}, \end{cases} \] which is exactly condition (CCL) of Subsection~\ref{subsec_ccl_update}. Hence, \(\diamond^{\mathcal F_{\mathrm{CCL}}}\) is a CCL update operator and, in particular, a maximal consistency-preserving SCL update operator.
\end{proof}

The preceding propositions, together with the semantic characterizations of CL and CCL belief update, yield the following inclusion relations between the corresponding classes of update operators.

\begin{corollary}\label{cor_cl_ccl_inclusions}
Every CL update operator is an SCL update operator, while every CCL update operator is a maximal consistency-preserving SCL update operator.
\end{corollary}

\begin{proof}
Let \(\diamond\) be a CL update operator. By Theorem~\ref{thm_cl_charact}, there exists a credible faithful assignment \(w\mapsto(C_w,\preceq_w)\) representing \(\diamond\) by condition (CL). Applying Proposition~\ref{prop_cl_limiting_case} to this assignment, with \(f_w^{\mathrm{CL}}(\varphi)=\varphi\), yields a pointwise transformation assignment \(\mathcal F_{\mathrm{CL}}\) satisfying properties \mbox{(F1)--(F4)}, such that condition (SCL) holds for \(\diamond\). Hence, every CL update operator is an SCL update operator.

Now, let \(\diamond\) be a CCL update operator. By Theorem~\ref{thm_ccl_charact}, there exists a credible faithful assignment \(w\mapsto(C_w,\preceq_w)\) representing \(\diamond\) by condition (CCL). Applying Proposition~\ref{prop_ccl_limiting_case} to this assignment yields a pointwise transformation assignment \(\mathcal F_{\mathrm{CCL}}\) satisfying properties (F1)--(F6), such that condition (SCL) holds for \(\diamond\). Hence, every CCL update operator is a maximal consistency-preserving SCL update operator.
\end{proof}

\subsection{KM Belief Update within the SCL Hierarchy}
\label{subsec_km_scl_hierarchy}

Corollary~\ref{cor_cl_ccl_inclusions} establishes that CL and CCL update operators arise as sub-classes of the class of SCL update operators. To complete the inclusion structure among the principal classes considered in this article, we now locate standard KM belief update within this hierarchy. KM belief update is recovered by taking every credible set to be the entire set of worlds and every transformation function to be the identity.

\begin{proposition} \label{prop_km_common_subclass}
Every KM update operator is a CL update operator and a maximal consistency-preserving SCL update operator.
\end{proposition}

\begin{proof}
Let \(\diamond\) be a KM update operator. By Theorem~\ref{thm_update_char}, there exists a faithful pointwise assignment \(w\mapsto\preceq_w\) such that $[K\diamond\varphi]=\displaystyle\bigcup_{w\in[K]}\min([\varphi],\preceq_w)$. For every world \(w\in\mathbb M\), put \(C_w=\mathbb M\) and define $f_w(\varphi)=\varphi$. Since \(C_w=\mathbb M\), condition (CL) reduces to $[K\diamond\varphi] = \displaystyle\bigcup_{w\in[K]} \min\bigl([\varphi],\preceq_w\bigr)$, which is exactly condition (U). Therefore, \(\diamond\) is a CL update operator.

Moreover, the identity transformation \(f_w(\varphi)=\varphi\) satisfies properties (F1)--(F4) and (F6). It also satisfies property (F5), since, whenever \([\varphi]\neq\varnothing\), $[f_w(\varphi)]\cap C_w = [\varphi]\cap\mathbb M = [\varphi] \neq\varnothing$. Finally, condition (SCL) reduces to $[K\diamond\varphi] = \displaystyle\bigcup_{w\in[K]} \min\bigl([\varphi],\preceq_w\bigr)$, which is again exactly condition (U). Therefore, \(\diamond\) is a maximal consistency-preserving SCL update operator.
\end{proof}

\begin{remark}
When the epistemic input \(\varphi\) is consistent and \(C_w=\mathbb M\), for every \(w\in\mathbb M\), condition (CCL) of Theorem~\ref{thm_ccl_charact} reduces to condition (U). Thus, under unrestricted credibility, KM and CCL belief update coincide on consistent epistemic inputs. This coincidence does not, however, amount to an inclusion between the corresponding classes of operators over the unrestricted domain of epistemic inputs.
\end{remark}

\subsection{Inclusion Hierarchy and Strictness}
\label{subsec_inclusion_hierarchy}

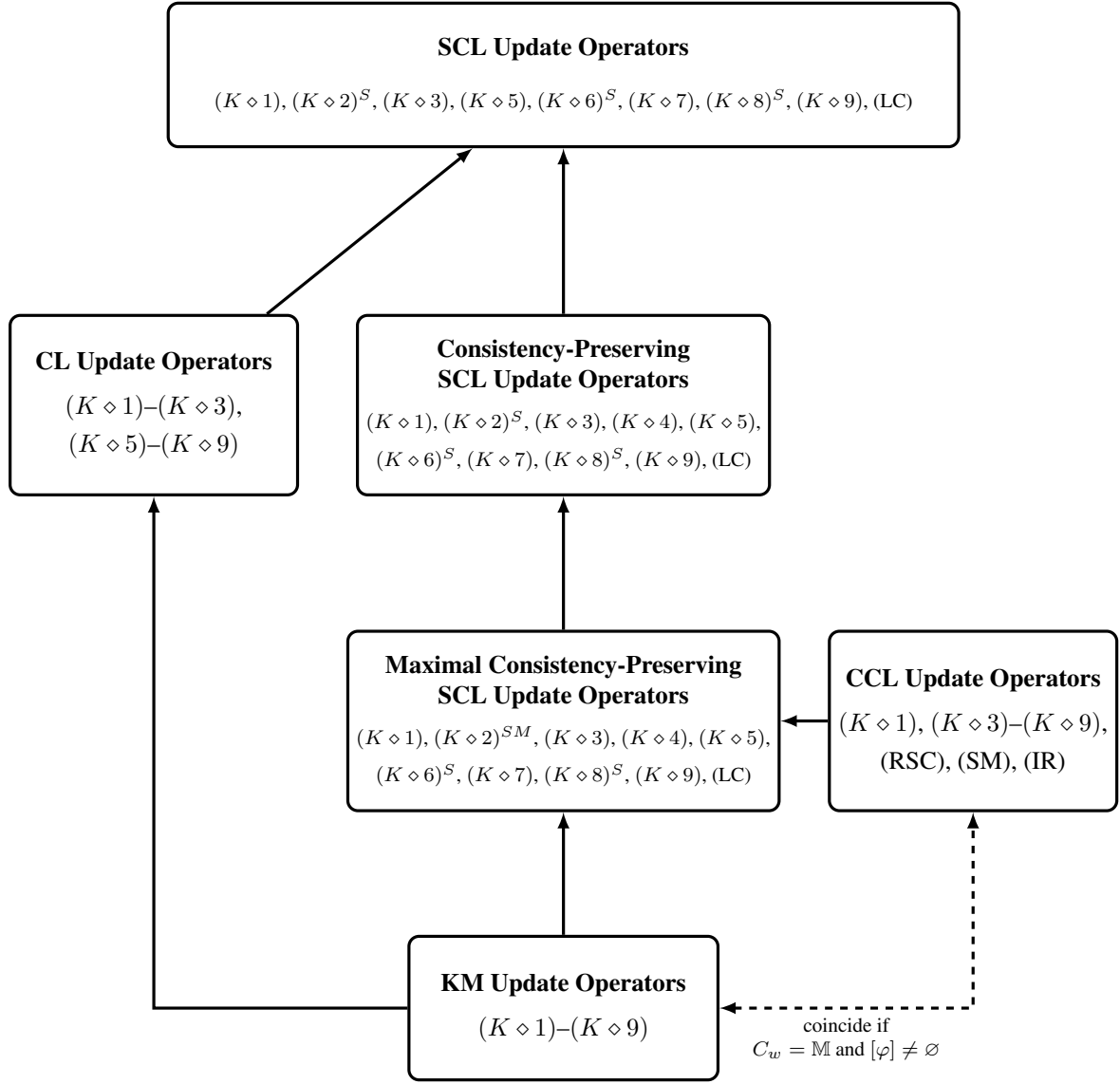
\begin{figure}[h!]
\centering\small
\begin{tikzpicture}[
box/.style={rectangle, draw, rounded corners, minimum width=4.3cm, minimum height=2cm, align=center},
inclusion/.style={->, very thick, >=latex},
coincidence/.style={<->, very thick, dashed, >=latex}
]
\node[box, very thick, minimum width=11cm] (scl) at (0,13) {{\bf SCL Update Operators}\\[3mm] {\scriptsize $(K\diamond1)$, $(K\diamond2)^{S}$, $(K\diamond3)$, $(K\diamond5)$, $(K\diamond6)^{S}$, $(K\diamond7)$, $(K\diamond8)^{S}$, $(K\diamond9)$, (LC)}};
\node[box, very thick, minimum height=2.5cm] (scscl) at (0,8.4) {{\bf Consistency-Preserving}\\{\bf SCL Update Operators}\\[1.5mm] {\scriptsize $(K\diamond1)$, $(K\diamond2)^{S}$, $(K\diamond3)$, $(K\diamond4)$, $(K\diamond5)$,}\\[1mm] {\scriptsize $(K\diamond6)^{S}$, $(K\diamond7)$, $(K\diamond8)^{S}$, $(K\diamond9)$, (LC)}};
\node[box, very thick, minimum height=2.5cm] (maxscl) at (0,4) {{\bf Maximal Consistency-Preserving}\\{\bf SCL Update Operators}\\[1.5mm] {\scriptsize $(K\diamond1)$, $(K\diamond2)^{SM}$, $(K\diamond3)$, $(K\diamond4)$, $(K\diamond5)$,}\\[1mm] {\scriptsize $(K\diamond6)^{S}$, $(K\diamond7)$, $(K\diamond8)^{S}$, $(K\diamond9)$, (LC)}};
\node[box, very thick, minimum height=2.5cm, minimum width=4cm] (cl) at (-5.7,8.4) {{\bf CL Update Operators}\\[2mm] $(K\diamond1)$--$(K\diamond3)$,\\[1mm] $(K\diamond5)$--$(K\diamond9)$};
\node[box, very thick, minimum height=2.5cm, minimum width=4cm] (ccl) at (5.7,4) {{\bf CCL Update Operators}\\[2mm] $(K\diamond1)$, $(K\diamond3)$--$(K\diamond9)$,\\[1mm] (RSC), (SM), (IR)};
\node[box, very thick] (km) at (0,0) {{\bf KM Update Operators}\\[2mm] $(K\diamond1)$--$(K\diamond9)$};

\draw[inclusion] (scscl) -- (scl);
\draw[inclusion] (maxscl) -- (scscl);
\draw[inclusion] (cl) -- (scl);
\draw[inclusion] (ccl.west) -- (maxscl.east);
\draw[inclusion] (km) -- (maxscl);
\draw[inclusion] (km.west) -| (cl.south);
\draw[coincidence]
(km.east) -|
node[pos=0.25, below, align=center, font=\scriptsize]
{coincide if\\
$C_w=\mathbb M$ and $[\varphi]\neq\varnothing$}
(ccl.south);
\end{tikzpicture}
\caption{Relations among the principal classes of SCL update operators. Solid arrows point from proper sub-classes to their super-classes. The dashed bidirectional arrow indicates that KM and CCL belief update coincide under unrestricted credibility for consistent epistemic inputs.}
\label{fig_map}
\end{figure}

The inclusion relations established above yield the hierarchy depicted in Figure~\ref{fig_map}. Maximal consistency-preserving SCL update operators form a sub-class of consistency-preserving SCL update operators, which in turn form a sub-class of SCL update operators. CL update operators form a sub-class of SCL update operators, whereas CCL update operators form a sub-class of maximal consistency-preserving SCL update operators. KM update operators form a common sub-class of CL update operators and maximal consistency-preserving SCL update operators. Moreover, under unrestricted credibility, KM and CCL belief update coincide whenever the epistemic input is consistent.

As a matter of fact, the inclusions depicted in Figure~\ref{fig_map} are all {\em proper}, as Proposition~\ref{prop_proper_inclusions} below proves. Hence, the proposed SCL framework offers a unified and strictly more expressive model of belief update, as it subsumes the established approaches, while also accommodating source-dependent selective acceptance.

\begin{proposition} \label{prop_proper_inclusions}
All inclusion relations represented by the solid arrows in
Figure~\ref{fig_map} are proper.
\end{proposition}

\begin{proof}
Ferm{\'e} \emph{et al.}~\cite{ferme23} show that expansion is a simple example of a CL update operator that is not a KM update operator. Since expansion may produce an inconsistent result from a consistent prior belief set and a consistent epistemic input, it is not a consistency-preserving SCL update operator. Hence, the class of KM update operators is a proper sub-class of the class of CL update operators, while the class of consistency-preserving SCL update operators is a proper sub-class of the class of SCL update operators.

The CCL update operator illustrated in Example~\ref{ex_cup_ccl} of Subsection~\ref{subsec_ccl_update} does not satisfy the success postulate \((K\diamond2)\), since \(\varphi\notin K\diamond_{\mathrm{CCL}}\varphi\). Therefore, it is neither a KM nor a CL update operator. Since every CCL update operator is a maximal consistency-preserving SCL update operator, this establishes that KM update operators form a proper sub-class of maximal consistency-preserving SCL update operators. It also shows that CL update operators form a proper sub-class of SCL update operators.

It remains to separate the class of CCL update operators from the class of maximal consistency-preserving SCL update operators. Consider a language built from \(\mathcal P=\{a,b\}\), and let the worlds $w=\bar a\bar b$ and $r=a\bar b$. Define a credible faithful assignment by setting \(C_w=\{w,r\}\), with \(w\prec_w r\), and \(C_u=\{u\}\), with \(\preceq_u=\{(u,u)\}\), for every \(u\in\mathbb M\setminus\{w\}\). For each \(u\in\mathbb M\), let \(\gamma_u\) be a complete sentence such that \([\gamma_u]=\{u\}\). Define the source-dependent transformation functions by
\[
f_w(\chi)=
\begin{cases}
\chi, & \text{if }\ [\chi]\cap C_w\neq\varnothing,\\
\chi\lor\gamma_r, & \text{otherwise},
\end{cases}
\]
and, for every \(u\neq w\), by
\[
f_u(\chi)=
\begin{cases}
\chi, & \text{if }\ u\in[\chi],\\
\chi\lor\gamma_u, & \text{otherwise}.
\end{cases}
\]
Let \(\mathcal F=\{f_u:u\in\mathbb M\}\), and let \(\diamond\) be the SCL update operator determined by the above credible faithful assignment and the pointwise transformation assignment $\mathcal F$, through condition (SCL). By the same argument as in Proposition~\ref{prop_ccl_limiting_case}, we can show that \(\mathcal F\) satisfies properties \mbox{(F1)--(F6)}. Hence, \(\diamond\) is a maximal consistency-preserving SCL update operator.

Now, let \(K=Cn(w)\) and \(\varphi=a\land b\). Since \([\varphi]\cap C_w=\varnothing\), we have that $f_w(\varphi)=\varphi\lor\gamma_r$ and $[f_w(\varphi)]\cap C_w=\{r\}$. Therefore, condition (SCL) gives $[K\diamond\varphi]=\{r\}$. Since \(r\not\models\varphi\) and \(r\neq w\), we have that \(\varphi\notin K\diamond\varphi\) and \(K\diamond\varphi\neq K\). Thus, postulate (RSC) of Definition~\ref{def_ccl_operator} fails, and \(\diamond\) is not a CCL update operator. Therefore, the class of CCL update operators is a proper sub-class of the class of maximal consistency-preserving SCL update operators.

Finally, to separate the class of maximal consistency-preserving SCL update operators from the class of consistency-preserving SCL update operators, consider a language built from \(\mathcal P=\{a,b\}\), and let the worlds $w=\bar a\bar b$, $r_1=a\bar b$, and $r_2=\bar a b$. Define a credible faithful assignment by setting $C_w=\{w,r_1,r_2\}$, where \(w\prec_w r_1\), \(w\prec_w r_2\), and \(r_1\) and \(r_2\) are incomparable with respect to $\preceq_w$. For every \(u\in\mathbb M\setminus\{w\}\), put \(C_u=\{u\}\) and \(\preceq_u=\{(u,u)\}\).

Define the transformation function associated with \(w\) by
\[
f_w(\chi)=
\begin{cases}
\chi, & \text{if }\ [\chi]\cap C_w\neq\varnothing,\\[1mm]
a\lor b, & \text{if }\ \chi\equiv a\land b,\\[1mm]
\top, & \text{otherwise}.
\end{cases}
\]
For every \(u\in\mathbb M\setminus\{w\}\), let \(\gamma_u\) be a complete sentence such that \([\gamma_u]=\{u\}\), and define
\[
f_u(\chi)=
\begin{cases}
\chi, & \text{if }u\in[\chi],\\[1mm]
\chi\lor\gamma_u, & \text{otherwise}.
\end{cases}
\]
Let \(\mathcal F=\{f_u:u\in\mathbb M\}\), and let \(\diamond\) be the SCL update operator determined by the above credible faithful assignment and the pointwise transformation assignment $\mathcal F$, through condition (SCL). It is readily verified that \(\mathcal F\) satisfies properties (F1)--(F5). Hence, \(\diamond\) is a consistency-preserving SCL update operator.

Now, let \(K=Cn(w)\) and \(\varphi=a\land b\). Since \([\varphi]\cap C_w=\varnothing\), we have that $f_w(\varphi)=a\lor b$ and $[f_w(\varphi)]\cap C_w=\{r_1,r_2\}$. Since \(r_1\) and \(r_2\) are incomparable with respect to $\preceq_w$, condition (SCL) gives $[K\diamond\varphi]=\{r_1,r_2\}$.

Suppose that a sentence \(\psi\in\mathcal{L}\) satisfies $\varphi\models\psi$, $\psi\in K\diamond\varphi$, and $K\diamond\varphi=K\diamond\psi$. Since \(\psi\in K\diamond\varphi\), both \(r_1\) and \(r_2\) satisfy \(\psi\); moreover, \(\varphi\models\psi\) entails that \(ab\models\psi\). As the language contains only the atoms \(a\) and \(b\), it follows that \(\psi\) is logically equivalent either to \(a\lor b\) or to \(\top\). The latter possibility is excluded because $[K\diamond\top]=\{w\}\neq\{r_1,r_2\}$. Hence, every possible witness \(\psi\) for postulate \((K\diamond2)^{SM}\) is logically equivalent to \(a\lor b\).

Consider now the sentence \(\chi=a\). We have that $\varphi\models a$ and $a\models a\lor b$, and, since \mbox{\([a]\cap C_w=\{r_1\}\)}, $[K\diamond a]=\{r_1\}$. Consequently, \(a\in K\diamond a\), while \(a\lor b\not\models a\). Thus, no possible witness \(\psi\) satisfies the maximality requirement of postulate \((K\diamond2)^{SM}\). Therefore, \(\diamond\) is not a maximal consistency-preserving SCL update operator. Hence, the class of maximal consistency-preserving SCL update operators is a proper sub-class of the class of consistency-preserving SCL update operators.

The aforementioned separating examples establish the strictness of every inclusion represented in Figure~\ref{fig_map}.
\end{proof}

\section{Conclusion}

This article introduced selective credibility-limited belief update, a non-prioritized framework that builds on three principal lines of research; namely, the pointwise semantics of Katsuno--Mendelzon (KM) belief update~\cite{katsuno92}, the credibility-limited (CL) and consistent credibility-limited (CCL) belief-update models of \mbox{Ferm{\'e} \emph{et al.}~\cite{ferme23}}, and the transformation-based approach developed for selective belief revision~\cite{ferme99}. The proposed framework combines these ideas by integrating source-dependent transformation functions into credibility-limited transition structures. For each initially possible world, the epistemic input is first replaced by a weaker proxy representing the information that can be accepted from that source world; the most plausible credible worlds of the resulting proxy are then selected as possible successors. This two-stage construction preserves the pointwise character of KM belief update, while overcoming the all-or-nothing treatment of compound epistemic inputs found in existing credibility-limited approaches~\cite{ferme23}.

We provided both semantic and axiomatic characterizations of the resulting class of selective credibility-limited (SCL) update operators. We also identified two progressively more constrained sub-classes. Consistency-preserving SCL update operators require every transformed epistemic input to have a credible world relative to its source world whenever the original epistemic input is consistent, thus ensuring that every initially possible world contributes at least one successor in such cases. Maximal consistency-preserving SCL update operators additionally require, for every consistent epistemic input, the selected proxy to be maximally informative among its credible consequences. The corresponding representation results clarify the relationship between the semantic properties imposed on source-dependent transformations and the observable rationality properties of the induced update operators.

The SCL framework also provides a unified perspective on the established approaches from which it originates. CL belief update is recovered by taking all source-dependent transformation functions to be identities, whereas CCL belief update is obtained by replacing locally non-credible epistemic inputs with proxies that retain the corresponding source worlds. Standard KM belief update is recovered by additionally removing the credibility restrictions. Moreover, the inclusion results establish that the principal classes considered herein are related by proper containment. Therefore, selective credibility-limited belief update provides a strictly more expressive framework than the established approaches it encompasses, as it supports source-dependent partial acceptance, without requiring the elimination or unchanged retention of source branches from which the complete epistemic input cannot be realized.

Several directions for future work arise from the proposed framework. First, further study could address more relaxed variants of the framework, including selective belief update based only on properties (F1)--(F3), as specified in Subsection~\ref{subsec_genuine_selective}, as well as intermediate settings obtained by weakening the full set of requirements (F1)--(F4). Additional properties of pointwise transformation assignments could also be investigated, drawing on the conditions studied for transformation functions in selective belief revision~\cite[p. 335]{ferme99}. In particular, source-relative counterparts of monotonicity, negation conditions, conjunctive and disjunctive distribution, and disjunctive factoring may yield additional well-behaved classes of SCL update operators and corresponding axiomatic characterizations. Related questions concern the existence and uniqueness of admissible proxies, as well as the computational complexity of determining them and computing the resulting updates. Further work may also address iterated belief update \cite{ferme2023,fang2024}, and extensions to richer logical settings~\cite{ribeiro13}, multi-agent settings~\cite{tamargo12,lorini21}, and action-based formalisms~\cite{delval94,shapiro11}.

\currentpdfbookmark{Declarations}{Declarations}
\section*{Declarations}

\noindent {\bf Competing Interests:} The authors have no relevant financial or non-financial interests to disclose.

\vspace{3mm}

\noindent {\bf Data Availability:} Data sharing is not applicable to this article as no datasets were generated or analysed during the current study.

\currentpdfbookmark{References}{References}\bibliographystyle{plain}
\bibliography{references}
\end{document}